\documentclass[11pt]{article}

\usepackage{amsmath,amssymb,amsfonts}
\usepackage{algorithmic}
\usepackage{graphicx}
\usepackage{textcomp}
\usepackage{amsbsy}
\usepackage{xstring}
\usepackage{amsthm}
\usepackage{enumitem}
\usepackage[T1]{fontenc}    
\usepackage{hyperref}       
\usepackage{url}            
\usepackage{booktabs}       
\usepackage{nicefrac}       
\usepackage{microtype}      
\usepackage{mathtools}
\usepackage{amsmath}
\usepackage{standalone}
\usepackage{textcomp}
\usepackage{float}
\usepackage{url}
\usepackage{mathabx}
\usepackage{tabularx}
\usepackage{pgfplots}
\pgfplotsset{compat=1.14}
\usepgfplotslibrary{groupplots}
\usepackage{adjustbox}
\usepackage{subcaption}
\usepackage{pgfplots}
\usepackage{pgffor}
\usepackage{xprintlen}
\usepackage{printlen}
\usepackage{mfirstuc}
\usepackage{mfirstuc}
\usepackage{authblk}
\usepackage{upgreek}
\usepackage{mathtools}
\usepackage{bm}
\usepackage[numbers]{natbib}
\usepackage{lineno}
\usepackage{xstring}
\usepackage{xspace}
\usepackage{chngcntr} 
\usepackage{multibib}

\usepackage{lineno}

\usepackage{xstring}
\usepackage{xspace}

\newcommand{\op}[1]{\mathbf{#1}}
\renewcommand{\vec}[1]{\mathbf{#1}}

\newcommand{\norm}[1]{\left|\left|#1\right|\right|}

\newcommand{\id}{\op{I}}
\newcommand{\fourier}{\op{F}}
\newcommand{\ifourier}{\op{F}^{-1}}

\newcommand{\rotop}{\op{R}}

\newcommand{\FBP}{\text{FBP}}
\newcommand{\sliceop}{\op{S}}
\newcommand{\ctf}{\op{C}}

\newcommand{\expectation}[1]{\mathop{\mathbb{E}}_{#1}}
\newcommand{\proj}{\vec{t}}
\newcommand{\ts}{\vec{t}}

\newcommand{\tsOne}{{\color{blue}\ts^\text{1}}}
\newcommand{\tsZero}{{\color{red}\ts^\text{0}}}
\newcommand{\network}{\mathrm{f}_\nnweights}

\newcommand{\selfLoss}{\ell}

\newcommand{\riskfun}{\mathrm{R}}
\newcommand{\lossfun}{\mathrm{L}}
\newcommand{\prob}{\mathrm{P}}

\newcommand{\vol}{\vec{v}}

\newcommand{\truevol}{\vol^*}
\newcommand{\recvol}{\hat{\vol}}

\newcommand{\volZero}[1]{{\color{red}\vol_{{\color{black}#1}}^\text{0}}}
\newcommand{\volZeros}{\{\volZero{i}\}_{i=1}^N}

\newcommand{\volOne}[1]{{\color{blue}\vol_{{\color{black}#1}}^\text{1}}}
\newcommand{\volOnes}{\{\volOne{i}\}_{i=1}^N}

\newcommand{\noisierVolZero}[1]{{\color{red}\tilde{\vol}_{{\color{black}#1}}^\text{0}}}
\newcommand{\noisierVolZeros}{\{\noisierVolZero{i,\anglesPhi_i}\}_{i=1}^N}

\newcommand{\struc}{\vec{v}}
\newcommand{\truestruc}[1]{\struc_{#1}^*}
\newcommand{\strucZero}[1]{{\color{red}\struc_{{\color{black}#1}}^\text{0}}}
\newcommand{\strucOne}[1]{{\color{blue}\struc_{{\color{black}#1}}^\text{1}}}
\newcommand{\noisierStrucZero}[1]{{\color{red}\tilde{\struc}_{{\color{black}#1}}^\text{0}}}

\newcommand{\crosscorr}{\text{CC}}
\newcommand{\mean}{\text{mean}}

\newcommand{\noise}{\vec{n}}
\newcommand{\noiseOne}{{\color{blue}\noise^\text{1}}}
\newcommand{\noiseZero}{{\color{red}\noise^\text{0}}}

\newcommand{\noisierMask}{{\color{black}\tilde{M}}}
\newcommand{\mask}{M}

\newcommand{\nnweights}{\text{$\boldsymbol{\theta}$}}
\newcommand{\anglesPhi}{\text{$\boldsymbol{\varphi}$}}

\newcommand{\reals}{\mathbb{R}}

\newtheorem{lem}{Lemma}
\newtheorem{prop}{Proposition}

\begin{document}

\title{A Deep Learning Method for Simultaneous Denoising and Missing Wedge Reconstruction in Cryogenic Electron Tomography}
\author[]{Simon Wiedemann}
\author[]{Reinhard Heckel
}

\affil[]{Department of Computer Engineering, Technical University of Munich \{simonw.wiedemann,~reinhard.heckel\}@tum.de.}

\date{}

\maketitle

\begin{abstract} 
    Cryogenic electron tomography is a technique for imaging biological samples in 3D. A microscope collects a series of 2D projections of the sample, and the goal is to reconstruct the 3D density of the sample called the tomogram. 
    Reconstruction is difficult as the 2D projections are noisy and can not be recorded from all directions, resulting in a missing wedge of information.
    Tomograms conventionally reconstructed with filtered back-projection suffer from noise and strong artifacts due to the missing wedge. 
    Here, we propose a deep-learning approach for simultaneous denoising and missing wedge reconstruction called DeepDeWedge.
    The algorithm requires no ground truth data and is based on fitting a neural network to the 2D projections using a self-supervised loss.
    DeepDeWedge is simpler than current state-of-the-art approaches for denoising and missing wedge reconstruction, performs competitively and produces more denoised tomograms with higher overall contrast.
\end{abstract}

\section{Introduction}
Cryogenic electron tomography (cryo-ET) is a powerful cryo-electron microscopy (cryo-EM) technique for obtaining 3D models of biological samples such as cells, viruses, and proteins. An important application of cryo-ET is visualizing biological macromolecules like proteins in situ, i.e., in their (close-to) native environment. Imaging in situ preserves biological context, which can greatly improve the understanding of the workings of macromolecules \cite{wan2016cryo}. 

In cryo-ET, the sample to be imaged is first prepared on a grid and then frozen. Next, a transmission electron microscope records a tilt series, which is a collection of 2D projections of the sample's 3D scattering potential or density. Each projection in the tilt series is recorded after tilting the sample for a number of degrees around the microscope's tilt axis. 

From this tilt series, a tomogram, i.e., a discretized estimate of the sample's 3D density, can be estimated using computational techniques. For this inverse problem, numerous approaches have been proposed. The most commonly used tomographic reconstruction technique is filtered back-projection (FBP) \cite{radermacher1988three, wan2016cryo, turk2020promise}. Two major obstacles limit the resolution and interpretability of tomograms reconstructed with FBP and similar methods:
    \begin{enumerate}
        \item \textbf{Noisy data:} The total electron dose during tilt series acquisition must be low because biological samples are sensitive to radiation damage. Thus, the individual projections of the tilt series have low contrast and a low signal-to-noise ratio (SNR).
        \item \textbf{A missing wedge of information:} The range of angles at which useful images can be collected is often limited to, for example, $\pm 60^\circ$  rather than the desired full range of $\pm 90^\circ$. This is due to the increased thickness of the sample in the direction of the electron beam for tilt angles of large magnitude~\cite{liu2022isotropic}. 
        The missing data in the tilt series is wedge-shaped in the Fourier domain.
    \end{enumerate}

To define the missing wedge and motivation for our work, we now state a model for the tilt series acquisition process in cryo-ET: A tilt series $\ts = \big(\proj_{-K},\ldots,\proj_{K}\big)$ is a collection of 2D projections, where each 2D projection
$\proj_k$ is a measurement of an underlying ground truth volume $\truevol$ obtained with the electron microscope. 
By the Fourier slice theorem (see e.g.\ \cite{malzbender1993fourier}), the 2D Fourier transform $\fourier \proj_k$ of each projection image $\proj_k$ of a tilt series is a noisy observation of a 2D central slice through true volume's 3D Fourier transform $\fourier\truevol$ multiplied with an additional filter, i.e.,
    \begin{equation}
    \label{eq:projection_image_formation_fourier}
        \fourier \proj_k = \ctf_k \cdot \sliceop \rotop_k \fourier \truevol + \fourier \vec{\noise}_k.
    \end{equation}
In Equation \eqref{eq:projection_image_formation_fourier}, the rotation operator 
$\rotop_k$
spatially rotates the 3D Fourier transform $\fourier \truevol$ of the volume $\truevol$ by the tilt angle $\alpha_k$ around the microscope's tilt axis. Then, the slice-operator 
$\sliceop$
extracts the central 2D slice of the volume's Fourier transform that is perpendicular to the microscope's optical axis. The filter 
$\ctf_k$ 
is the contrast transfer function (CTF) of the microscope and models optical aberrations. The term $\fourier \vec{\noise}_k$ is the Fourier transform of a random 2D image-domain noise term 
$\vec{\noise}_k$. 
It is often assumed that the noise $\vec{\noise}_k$ comes from a Poisson distribution (shot noise). Another common assumption is that the Fourier-domain noise $\fourier \vec{\noise}_k$ is Gaussian, but not necessarily white \cite{bendory2020single, frangakis2021s}. 
In this work, we make the weaker assumption that the 2D noise terms $\vec{\noise}_{k}$ and $\vec{\noise}_\ell$ of any two distinct projections indexed with $k$ and $\ell$ ($k \neq \ell$) have zero means and are independent. 

Equation \eqref{eq:projection_image_formation_fourier} is commonly used as a Fourier-domain model of the image formation process in cryo-EM \cite{sigworth2016principles, bendory2020single}. As the Fourier slice theorem assumes continuous representations of volumes and images, this model is exact only in the continuous case. In this work, we consider discrete representations of volumes and images, as is common in cryo-EM practice. Thus, Equation \eqref{eq:projection_image_formation_fourier} holds only approximately.

The range of tilt angles $\alpha_k$ that yield useful projections $\proj_k$ is typically limited to, e.g., $\pm 60^\circ$ rather than the full range $\pm 90^\circ$. Therefore, by Equation \eqref{eq:projection_image_formation_fourier} there is a wedge-shaped region of the sample's Fourier representation $\fourier \truevol$ which is not covered by any of the Fourier slices $\fourier \proj_k$. 

We consider the problem of reconstructing a tomogram, i.e., a discretized estimate of a sample's density from a noisy, incomplete tilt series. This is a challenging inverse problem due to the high noise level and the missing wedge. 

Recently, deep-learning-based methods for denoising and missing wedge reconstruction have been proposed. However, these are effective for denoising and missing wedge reconstruction individually but not simultaneously. Specifically, IsoNet~\cite{liu2022isotropic}, a closely related work, 
does well at missing wedge reconstruction, but its denoising performance is low compared to state-of-the-art denoising methods \cite{maldonado2023fourier}. The current state-of-the-art approaches for denoising in cryo-ET build on Noise2Noise \cite{lehtinen2018noise2noise}, a framework for deep-learning-based image denoising.
Popular software packages that implement Noise2Noise-based denoising methods for cryo-ET tomograms are CryoCARE~\cite{buchholz2019cryo}, Topaz~\cite{bepler2020topaz} and Warp~\cite{tegunov2019real}.

In this paper, we propose DeepDeWedge, a deep learning-based approach for tomogram reconstruction that simultaneously performs well on denoising and missing wedge reconstruction. 

DeepDeWedge takes one (or more; see Section \ref{sec:our_approach:algorithm}) tilt series as input and aims to estimate a noise-free 3D reconstruction of the samples' density with a filled-in missing wedge. 
To achieve this, we propose fitting a randomly initialized network with a self-supervised loss for simultaneous denoising and missing wedge reconstruction. After fitting, the network is applied to the same data to estimate the tomogram. DeepDeWedge only uses the tilt series of the density we wish to reconstruct and no other training data. DeepDeWedge is most related to Noise2Noise-based denoising approaches and to IsoNet, both introduced above. We discuss the exact relations in Section \ref{sec:our_approach:related_work} after describing our algorithm. 
DeepDeWedge performs on par with IsoNet on a pure missing wedge reconstruction problem and achieves state-of-the-art denoising performance on a pure denoising problem.
Moreover, we find that the performance of DeepDeWedge for the joint denoising and wedge reconstruction problem is similar to the two-step approach of applying IsoNet to tomograms denoised with a state-of-the-art Noise2Noise-like denoiser. Moreover, DeepDeWedge is simpler and requires fewer hyperparameters to tune than the two-step approach.

\begin{figure}[!t]
    \centering
    \includegraphics[width=\textwidth]{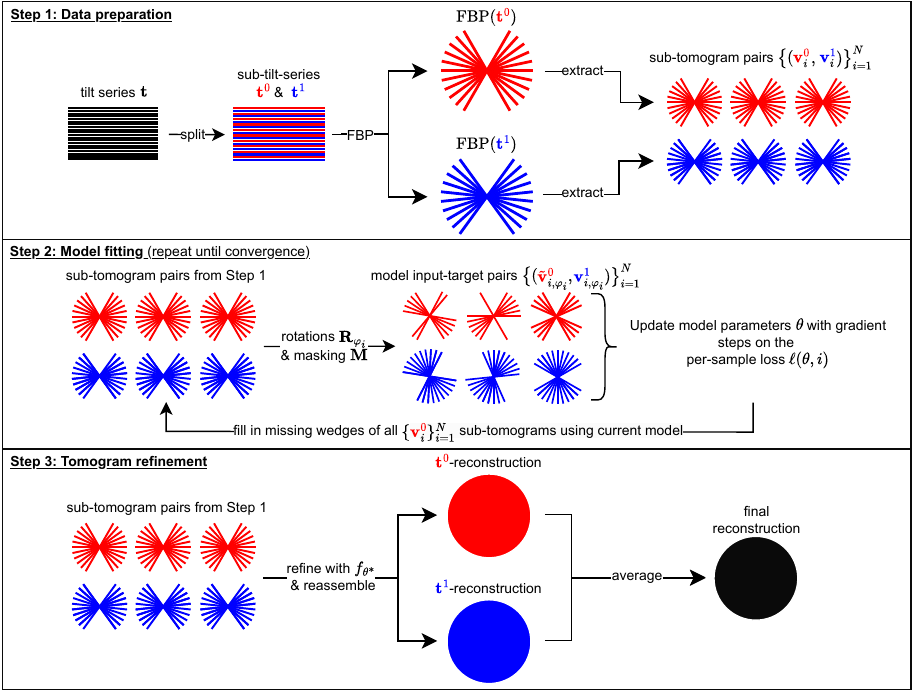}
    \caption{
    Illustration of DeepDeWedge. For simplicity, we show the 2D tilt series images as 1D Fourier slices and all 3D tomograms as 2D objects in the Fourier domain. Recall that tilt series images, tomograms, and sub-tomograms are objects in the image domain. The figure shows the splitting approach where the tilt series is split into even and odd projections.
    }

    \label{fig:our_approach}
\end{figure}

\section{DeepDeWedge: Denoising and Missing Wedge Reconstruction Algorithm}
\label{sec:our_approach:algorithm}

DeepDeWedge takes a single tilt series $\ts$ as input and produces a denoised, missing-wedge-filled tomogram. The method can also be applied to a dataset containing multiple (typically up to 10) tilt series from different samples of the same type, for example, sections of different cells, which share the same cell type. DeepDeWedge consists of the following three steps, illustrated in Figure \ref{fig:our_approach}:

\begin{enumerate}
\item \textbf{Data preparation:}
First, split the tilt series $\ts$ into two sub-tilt-series $\tsZero$ and $\tsOne$ with the even/odd split or the frame-based split. The even/odd split partitions the tilt series into even and odd projections based on their order of acquisition. 
The frame-based split can be applied if the tilt series is collected using dose fractionation and entails averaging only the even and odd frames recorded at each tilt angle. 
We recommend the frame-based splitting approach whenever possible. 
After splitting, we have two sub-tilt-series.

Next, reconstruct both sub-tilt-series independently with FBP and apply CTF correction. This yields a pair of two coarse reconstructions $\big( \FBP(\tsZero), \FBP(\tsOne) \big)$ of the sample's 3D density. 
from both FBP reconstructions. The size and number $N$ of these sub-tomogram cubes is a hyperparameter. Experiments on synthetic data presented in the supplementary information
suggest that larger sub-tomograms tend to yield better results up to a point. 

\item \textbf{Model fitting:}
Fit a randomly initialized network $\network$, we use a U-Net \cite{ronneberger2015u}, with weights $\nnweights$ by repeating the following steps until convergence:
    \begin{enumerate}
        \item \textit{Generate model inputs and targets:} For each of the sub-tomogram pairs  $\{(\volZero{i},\volOne{i})\}_{i=1}^N$ generated in Step 1, sample a rotation $\rotop_{\anglesPhi_i}$ parameterized by Euler angles $\anglesPhi_i$ from the uniform distribution on the group of 3D rotations, and construct a model input $\noisierVolZero{i, \anglesPhi_i}$ and target $\volOne{i, \anglesPhi_i}$ 
        by applying the rotation $\rotop_{\anglesPhi_i}$ to both sub-tomograms and adding an artificial missing wedge to the rotated sub-tomogram $\rotop_{\anglesPhi_i} \volZero{i}$, as shown in the center panel of Figure \ref{fig:our_approach}. The missing wedge is added by taking the Fourier transform of the rotated sub-tomograms and multiplying them with a binary 3D mask $\op{\mask}$ that zeros out all Fourier components that lie inside the missing wedge.  
        Repeating this procedure for all sub-tomogram pairs $\{(\volZero{i},\volOne{i})\}_{i=1}^N$ yields a set of $N$ triplets consisting of model input, target sub-tomogram, and angle $\big\{ \big(\noisierVolZero{i,\anglesPhi_i}, \volOne{i,\anglesPhi_i},\anglesPhi_i\big)\}_{i=1}^N$.

        \item \textit{Update the model:}
        Update the model weights $\nnweights$ by performing gradient steps to minimize the per-sample loss 
                \begin{equation}
                    \label{eq:self_objective}
                    \selfLoss\left( \nnweights,i
                    \right) 
                    = 
                    \norm{ \left( \op{\mask} \op{\mask}_{\anglesPhi_i} + 2\op{\mask}^C \op{\mask}_{\anglesPhi_i} \right) \fourier \left( \network \big( \noisierVolZero{i,\anglesPhi_i} \big) - \volOne{i,\anglesPhi_i} \right)}_2^2.
                \end{equation}
            Here, $\op{\mask}_{\anglesPhi_i}$ is the rotated version of the wedge mask $\op{\mask}$ and $\op{\mask}^C := \op{I} - \op{M}$ is the complement of the mask.
            For the gradient updates, we use the Adam optimizer \cite{kingma2015adam}~and perform a single pass through the $N$ model input and target sub-tomograms generated before.
        
        \item \textit{Update the missing wedges of the model inputs:} For each $i =1,\ldots,N$, update the missing wedge the i-th sub-tomogram $\volZero{i}$ produced in Step 1 by passing it through the current model $\network$, and inserting the predicted content of the missing wedge, as follows
            \begin{equation}
                \volZero{i} \leftarrow \ifourier \Big( \op{\mask} \fourier \volZero{i} + \op{\mask}^C \fourier \network\left( \volZero{i} \right)   \Big).
            \end{equation}
        In the next input-target generation step, the model inputs $\noisierVolZeros$
        are constructed using the updated sub-tomograms. We do not update the missing wedges of the sub-tomograms 
        $\{\volOne{i}\}_{i=1}^N$
        used to generate the model targets, since their missing wedges are masked out in the per-sample loss, c.f. Equation \eqref{eq:self_objective}, and therefore play no role in model fitting.
    \end{enumerate}

\item \textbf{Tomogram refinement:}
Pass the original, non-updated sub-tomograms $\volZeros$
through the fitted model $f_{\nnweights^*}$ from Step 2, and reassemble the model outputs $\{f_{\nnweights^*}\left(\volZero{i}\right)\}_{i=1}^N$
into a full-sized tomogram. Repeat the same for the sub-tomograms $\volOnes$. Finally, average both reconstructions to obtain the final denoised and missing-wedge corrected tomogram. 
\end{enumerate}

\subsection{Motivation for the Three Steps of the Algorithm}
\label{sec:our_approach:motivation}

We now provide a brief motivation for each of the three steps of DeepDeWedge. 

\paragraph{Motivation for Step 1.}
In Step 1, we split the tilt series into two disjoint parts 
to obtain measurements with independent noise. As the noise on the individual projections or frames is assumed to be independent, the reconstructions $\FBP(\tsZero)$ and $\FBP(\tsOne)$ are noisy observations of the same underlying sample with independent noise terms. 
Those are used in Step 2 for the self-supervised Noise2Noise-inspired loss. 

Tilt series splitting is also used in popular implementations of Noise2Noise-like denoising methods for cryo-ET \cite{buchholz2019cryo, bepler2020topaz, tegunov2019real}. The frame-based splitting procedure was proposed by Buchholz et al.\ \cite{buchholz2019cryo}, who found that it can improve the performance of Noise2Noise-like denoising over the even/odd split. 

\paragraph{Motivation for Step 2.}
Step 2 of DeepDeWedge is to fit a neural network to perform denoising and missing wedge reconstruction, for which we have designed a specific loss function $\ell$ (Equation \eqref{eq:self_objective}). We provide a brief justification for the loss function here; a detailed theoretical motivation is presented in the following section.

As the masks $\op{\mask} \op{\mask}_{\anglesPhi_i}$, and $\op{\mask}^C \op{\mask}_{\anglesPhi_i}$ are orthogonal, the loss value $\selfLoss(i,\nnweights)$ can be expressed as the sum of two the two terms $\norm{ \op{\mask} \op{\mask}_{\anglesPhi_i} \fourier \left( \network \left( \noisierVolZero{i, {\anglesPhi_i}} \right) - \volOne{i,\anglesPhi_i} \right)}_2^2$, and $\norm{2\op{\mask}^C \op{\mask}_{\anglesPhi_i} \fourier \left( \network \left( \noisierVolZero{i,\anglesPhi_i} \right) - \volOne{i,\anglesPhi_i} \right)}_2^2$. The first summand, i.e., $\norm{ \op{\mask} \op{\mask}_{\anglesPhi_i} \fourier \left( \network \left( \noisierVolZero{i,\anglesPhi_i} \right) - \volOne{i,\anglesPhi_i} \right)}_2^2$, is the squared L2 distance between the network output and the target sub-tomogram $\volOne{i,\anglesPhi_i}$ on all Fourier components that were not masked out by the two missing wedge masks $\op{\mask}$ and $\op{\mask}_{\anglesPhi_i}$. As we assume the noise in the target to be independent of the noise in the input $\noisierVolZero{i,\anglesPhi_i}$, minimizing this part incentivizes the network to learn to denoise these Fourier components. This is inspired by the Noise2Noise principle. 

The second summand, i.e., $\norm{2\op{\mask}^C \op{\mask}_{\anglesPhi_i} \fourier \left( \network \left( \noisierVolZero{i,\anglesPhi_i} \right) - \volOne{i,\anglesPhi_i} \right)}_2^2$, incentivizes the network $\network$ to restore the data that we artificially removed with the mask $\op{\mask}$, and can be considered as a Noisier2Noise-like loss \cite{moran2020noisier2noise} (see Supplementary Information \ref{appx:background:noisier2noise}
for background).   
For this part, it is important that we rotate both volumes, which moves their original missing wedges to a new, random location.

In the last part of Step 2, we correct the missing wedges of the sub-tomograms $\volZeros$ using the current model $\network$. Therefore, as the model fitting proceeds, the model inputs $\noisierVolZeros$ will more and more resemble sub-tomograms with only one missing wedge, i.e., the one we artificially remove from the partially corrected sub-tomograms $\volZeros$ with the mask $\op{\mask}_{\anglesPhi_i}.$ This is a heuristic which is supposed to help the model perform well on the original sub-tomograms $\volZeros$ and $\volOnes$, which have only one missing wedge and which we use as model inputs in Step 3. An analogous approach for Noisier2Noise-based image denoising was proposed by Zhang et al.\ \cite{zhang2022idr}.

\paragraph{Motivation for Step 3.}
In Step 3, we use the fitted model from Step 2 to produce the final denoised and missing wedge filled tomogram. To use all the information contained in the tilt series $\ts$ for the final reconstruction, we separately refine the sub-tomograms from the FBP reconstructions $\FBP\left(\tsZero\right)$ and $\FBP\left(\tsOne\right)$ of the sub-tilt-series $\tsZero$ and $\tsOne$, and average them. For this, we apply a special normalization to the model inputs, which is described in Supplementary Information \ref{appx:normalization}.

\subsection{Related Work}
\label{sec:our_approach:related_work}

DeepDeWedge builds on Noise2Noise-based denoising methods and is related to the denoising and missing wedge-filling method IsoNet \cite{liu2022isotropic}. We first discuss the relation between DeepDeWedge and Noise2Noise-based methods, which do not reconstruct the missing wedge.

\paragraph{Relation to Noise2Noise-Like Denoising in Cryo-ET.}
Noise2Noise-based denoising algorithms for cryo-ET as implemented in CryoCARE \cite{buchholz2019cryo} or Warp \cite{tegunov2019real} take one or more tilt series as input and return denoised tomograms. A randomly initialized network $\network$ is fitted for denoising on sub-tomograms of FBP reconstructions $\FBP(\tsZero)$ and $\FBP(\tsOne)$ of sub-tilt-series $\tsZero$ and $\tsOne$ obtained from a full tilt series $\ts$. 
The model is fitted by minimizing a loss function (typically the mean-squared error), between the output of the model $\network$ applied to one noisy sub-tomogram and the corresponding other noisy sub-tomogram. The fitted model is then used to denoise the two reconstructions $\FBP(\tsZero)$ and $\FBP(\tsOne)$, which are then averaged to obtain the final denoised tomogram. Contrary to those denoising methods, DeepDeWedge fits a network not only to denoise, but also to fill the missing wedge. 

\paragraph{Relation to IsoNet.}
Our method is most closely related to IsoNet, which can also do denoising and missing wedge reconstruction. 
IsoNet takes a small set of, say, one to ten tomograms and produces denoised, missing-wedge corrected versions of those tomograms. 
Similar as for DeepDeWedge, a randomly initialized network is fitted for denoising and missing wedge reconstruction on sub-tomograms of these tomograms, however the fitting process is different: Inspired by Noisier2Noise, the model is fitted on the task of mapping sub-tomograms that are further corrupted with an additional missing wedge and additional noise onto their non-corrupted versions. After each iteration, the intermediate model is used to predict the content of the original missing wedges of all sub-tomograms. The predicted missing wedge content is inserted into all sub-tomograms, which serve as input to the next iteration of the algorithm. 

Different to IsoNet, our denoising approach is Noise2Noise-like, as in CryoCARE. This leads to better denoising performance, as we will see later, as well as requiring fewer assumptions and no hyperparameter tuning. Specifically, Noisier2Noise-like denoising requires knowledge of the noise model and strength (see Supplementary Information \ref{appx:background:noisier2noise}).
As this knowledge is typically unavailable, Liu et al.\ \cite{liu2022isotropic} propose approximate noise models from which the user has to choose. 

After model fitting, the user must manually decide which iteration and noise level gave the best reconstruction. Thus, IsoNet's Noisier2Noise-inspired denoising approach requires several hyperparameters for which good values exist but are unknown. Therefore, IsoNet requires tuning to achieve good results. 
Our denoising approach introduces no additional hyperparameters and does not require knowledge of the noise model and strength.

The main commonality between DeepDeWedge and IsoNet is the Noisier2Noise-like mechanism for missing wedge reconstruction, which consists of artificially removing another wedge from the sub-tomograms and fitting the model to reconstruct the wedge. 

Moreover, like IsoNet, DeepDeWedge fills in the missing wedges of the model inputs. In IsoNet, one also has to fill in the missing wedges of the model targets. This is necessary because, contrary to our loss $\selfLoss$ form Equation \eqref{eq:self_objective}, IsoNet's loss function does not ignore the targets' missing wedges via masking in the Fourier domain. 

\paragraph{Relation to Other Work.}
Another line of works related to DeepDeWedge considers domain-specific tomographic reconstruction methods that incorporate prior knowledge of biological samples into the reconstruction process to compensate for missing wedge artifacts, for example, ICON \cite{deng2016icon}, and MBIR \cite{yan2019mbir}. For an overview of such reconstruction methods, we refer to the introductory sections of works by Ding et al.\ \cite{ding2019joint} and Bohning et al.\ \cite{bohning2022compressed}. Liu et al.\ \cite{liu2022isotropic} found that IsoNet outperforms both ICON and MBIR.

DeepDeWedge is also conceptually related to un-trained neural networks, which reconstruct an image or volume based on fitting a neural network to data~\cite{ulyanov_DeepImagePrior_2018,heckel_DeepDecoderConcise_2019}. Un-trained networks also only rely on fitting a neural network to given measurements. However, they rely on the bias of convolutional neural networks towards natural images~\cite{heckel_DenoisingRegularizationExploiting_2020,heckel_CompressiveSensingUntrained_2020}, whereas in our setup, we fit a network on measurement data to be able to reconstruct from the same measurements. 

For cryo-EM-related problems other than tomographic reconstruction, deep learning approaches for missing data reconstruction and denoising have also recently been proposed. Zhang et al.\ \cite{zhang2023method} proposed a method to restore the state of individual particles inside tomograms, and Liu et al.\ \cite{liu2023resolving} proposed a variant of IsoNet to resolve the preferred orientation problem in single-particle cryo-EM \cite{liu2023resolving}. 

\subsection{Theoretical Motivation for the Loss Function}

\begin{figure}[!t]
    \centering
    \includegraphics{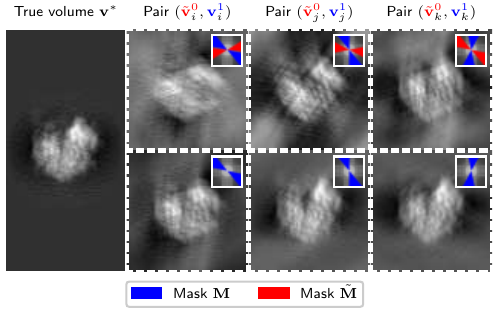}
    \caption{
    Illustration of a single ground-truth structure $\truestruc{}$, and three model input-target pairs corresponding to three different random positions of the original (blue) and additional (red) missing wedge for the setup of our theoretical motivation. 
    For simplicity, we visualize a 2D structure with no additive noise and random in-plane rotations. 
    The inset boxes in each patch show the absolute values of the Fourier transforms of the images and regions that are zeroed out by the missing wedge masks.
    }
    \label{fig:theory_setup}
\end{figure}

Here, we present a theoretical result that motivates the choice of our per-sample loss $\selfLoss$ defined in Equation \eqref{eq:self_objective}. The discussion in this section does not consider the heuristic of updating the missing wedges of the model inputs, which is part of DeepDeWedge's model fitting step. Moreover, we consider an idealized setup that deviates from practice to motivate our loss.

We assume access to data that consists of many noisy, missing-wedge-affected 3D observations of a fixed ground truth 3D structure $\truestruc{} \in \mathbb{R}^{N \times N \times N}$. 
Specifically, data is generated as two measurements (in the form of volumes) of the unknown ground-truth volume of interest, $\truestruc{}$, as 
    \begin{equation}
        \label{eq:prop_data}
        \strucZero{} = \ifourier \op{\mask} \fourier\big( \truestruc{} + \vec{\noiseZero}\big) ,
        \qquad
        \strucOne{} = \ifourier \op{\mask} \fourier \big( \truestruc{} + \vec{\noiseOne}\big),
    \end{equation}
where $\vec{\noiseZero}, \vec{\noiseOne} \in \mathbb{R}^{N \times N \times N}$ are random noise terms and $\op{\mask} \in \{0,1\}^{N \times N\times N}$ is the missing wedge mask. 
From the first measurement, we generate a noisier observation $\noisierStrucZero{} = \ifourier \op{\noisierMask} \fourier \strucZero{}$ by applying a second missing wedge mask $\op{\noisierMask}$. The noisier observation has two missing wedges: the wedge introduced by the first and the wedge introduced by the second mask. 
We assume that the two masks follow a joint and symmetric distribution, e.g., that for each mask, a random wedge is chosen uniformly at random. Figure \ref{fig:theory_setup} illustrates three data points in 2D.

We then train a neural network $\network$ to minimize the loss
\begin{align}
    \label{eq:self_objective_prop}
                \lossfun (\nnweights) = \expectation{\op{\mask}, \op{\noisierMask}, \noiseZero, \noiseOne} 
                \left[
                \norm{ \big( \op{\noisierMask} \op{\mask} +2 \op{\noisierMask}^C \op{\mask}\big) \fourier \Big( \network \big( \noisierStrucZero{} \big) - \strucOne{}\Big)}_2^2 
                \right],
\end{align}
where the expectation is over the random masks and the noise term. Note that this resembles training on infinitely many data points, with a very similar loss than the original loss~\eqref{eq:self_objective}; the main difference is that in the original loss, the volume is rotated randomly but the mask $\op{\mask}$ is fixed, while in the setup considered in this section, the volume is fixed but the masks $\op{\mask}$ and $\op{\noisierMask}$ are random.

After training, we can use the network to estimate the ground-truth volume by applying the network to another noisy observation $\noisierStrucZero{}$. The following proposition, whose proof can be found in Supplementary Information \ref{appx:theory},
establishes that this is equivalent to training the network on a supervised loss to reconstruct the input $\noisierStrucZero{}$, provided the two masks are non-overlapping. 

\begin{prop}
\label{prop:selfsupervised_is_supervised}
Assume that the noise $\vec{\noiseOne}$ is zero-mean and independent of the noise $\vec{\noiseZero}$, and of the masks $(\op{\mask},\op{\noisierMask})$, and assume that the noise $\vec{\noiseZero}$ is also independent of the masks. Moreover, assume that the joint probability distribution $P$ of the missing wedge masks $\op{\mask}$ and $\op{\noisierMask}$ is symmetric, i.e., $\prob(\op{\mask}, \op{\noisierMask}) = \prob(\op{\noisierMask}, \op{\mask})$, and that the missing wedges do not overlap. Then the loss $\lossfun$ is proportional to the supervised loss

            \begin{align}
            \label{eq:sup_objective_prop}
                \riskfun(\nnweights) = 
                \expectation{\op{\mask}, \op{\noisierMask}, \noiseZero} \left[ \norm{  \network\big(\vec{\noisierStrucZero{}}\big) - \vec{\truestruc} }_2^2 \right],
            \end{align}
        i.e., $
                \lossfun (\nnweights) = \riskfun(\nnweights) + c,
            $
        where $c$ is a numerical constant independent of the network parameters $\nnweights$.
    \end{prop}

In practice, we do not apply our approach to the problem of reconstructing a single fixed structure $\truestruc{}$ from multiple pairs of noisy observations with random missing wedges. Instead, we consider the problem of reconstructing several unique biological samples using a small dataset of tilt series. To this end, we fit a model with an empirical estimate of a risk similar to the one considered in Proposition \ref{prop:selfsupervised_is_supervised}. We fit the model on sub-tomogram pairs extracted from the FBP reconstructions of the even and odd sub-tilt-series, which exhibit independent noise. Moreover, as already mentioned above, in the setup of our algorithm, the two missing wedge masks $\op{\mask}$ and $\op{\noisierMask}$ themselves are not random. However, as we randomly rotate the model input sub-tomograms during model fitting, the missing wedges appear at a random location with respect to an arbitrary fixed orientation of the sub-tomogram. 

\section{Experiments}
\label{sec:experiments}

In the following sections, we compare DeepDeWedge to a re-implementation of CryoCARE, IsoNet, and a two-step approach of fitting IsoNet to tomograms denoised with CryoCARE. The two-step approach, which we call CryoCARE + IsoNet, is considered a state-of-the-art pipeline for denoising and missing wedge reconstruction. 

We implemented DeepDeWedge in Python, using PyTorch \cite{paszke2019pytorch} as a deep learning framework. For IsoNet, we used the authors' original implementation (\url{https://github.com/IsoNet-cryoET/IsoNet}), and for CryoCARE, we used our own re-implementation, which we based on our implementation of DeepDeWedge.

As model architecture,  we used a U-Net \cite{ronneberger2015u} with 64 channels and 3 downsampling layers for all three methods. The network has $27.3$ million trainable parameters and is also the default model architecture of IsoNet. 

IsoNet uses regularization with dropout with probability $0.3$, which we also kept for CryoCARE to prevent overfitting. For DeepDeWedge, we found that dropout is unnecessary, so we did not use dropout for our method unless explicitly stated otherwise.

We fitted all models using the Adam optimizer \cite{kingma2015adam} with a constant learning rate of $4\cdot10^{-4}$.

\subsection{Experiments on Purified \textit{S. Cerevisiae} 80S Ribosomes}
\label{sec:experiments:real:empiar10045}
\begin{figure}[!t]
    \centering
    \includegraphics[width=\textwidth]{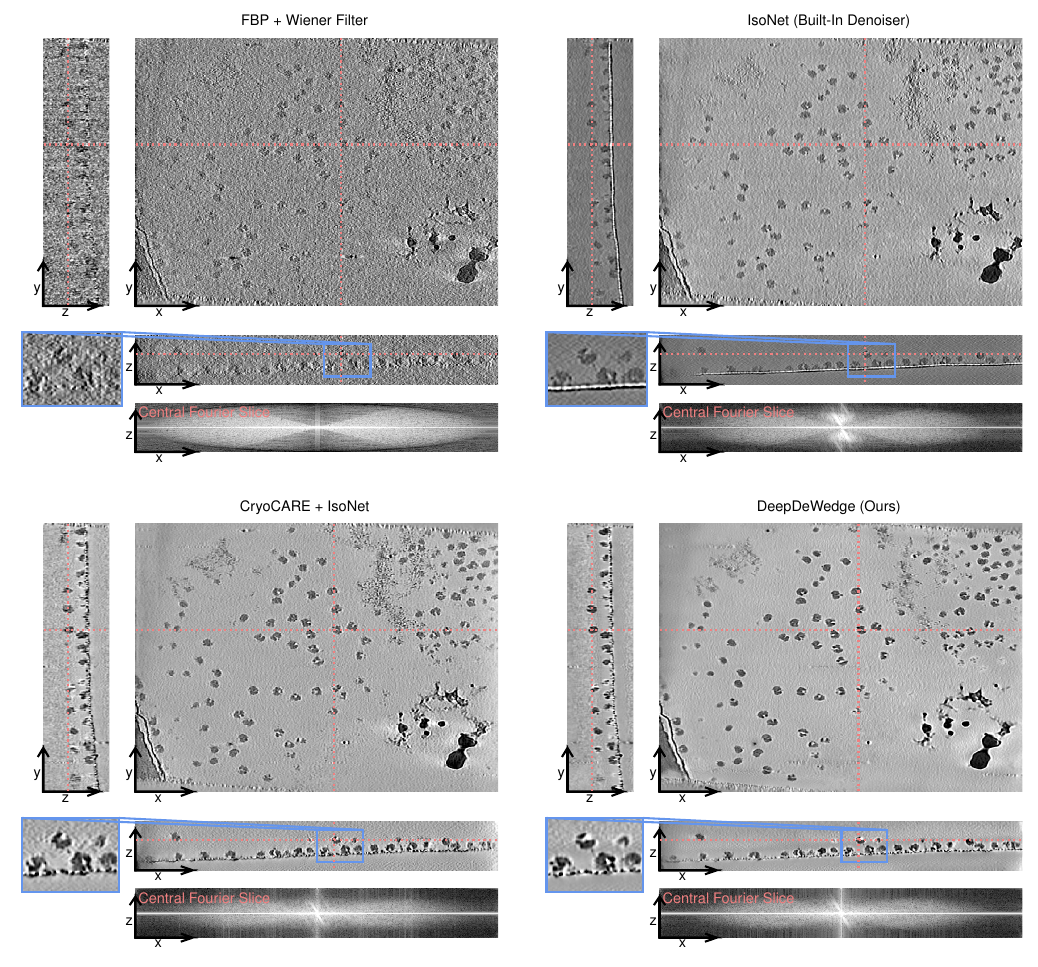}
    \caption{
    Slices through 3D reconstructions of a tomogram containing purified \textit{S. cerevisiae} 80S ribosomoes (EMPIAR-10045, Tomogram 5). 
    The red lines in each slice indicate the positions of the remaining two slices.
    We also show the central x-z-slice through the logarithm of the magnitude of the Fourier transform of each reconstruction.
    }
    \label{fig:empiar10045_vols}
\end{figure}

The first dataset we consider is the commonly used EMPIAR-10045 dataset,
which contains 7 tilt series collected from samples of purified \textit{S. cerevisiae} 80S Ribosomes. More details on the data and the model fitting can be found in Supplementary Information \ref{appx:experiments:real:empiar10045}.

\paragraph{Results.} Figure \ref{fig:empiar10045_vols} shows a tomogram refined with IsoNet, CryoCARE + IsoNet and DeepDeWedge using the even/odd tilt series splitting approach. Note that while IsoNet's built-in Noisier2Noise-like denoiser removes some of the noise contained in the FBP reconstruction, its performance is considerably worse than that of Noise2Noise-based CryoCARE. 
This can be seen by comparing to the result of applying IsoNet with disabled denoiser to a tomogram denoised with CryoCARE. 
DeepDeWedge produces a denoised and missing-wedge-corrected tomogram similar to the CryoCARE + IsoNet combination. The main difference between these reconstructions is that the DeepDeWedge-refined tomogram has a smoother background and contains fewer high-frequency components. 

Regarding missing wedge correction, we find the performance of DeepDeWedge and IsoNet to be similar. In slices parallel to the x-z-plane, where the effects of the missing wedge on the FBP reconstruction are most prominent, both IsoNet and DeepDeWedge reduce artifacts and correct artificial elongations of the ribosomes.
The central x-z-slices through the reconstructions' Fourier transforms confirm that all methods but FBP fill in most of the missing wedge.

\subsection{Experiments on Flagella of \textit{C. Reinhardtii}}
\label{sec:experiments:real:tomo110}

\begin{figure}[t]
    \centering
    \includegraphics[width=\linewidth]{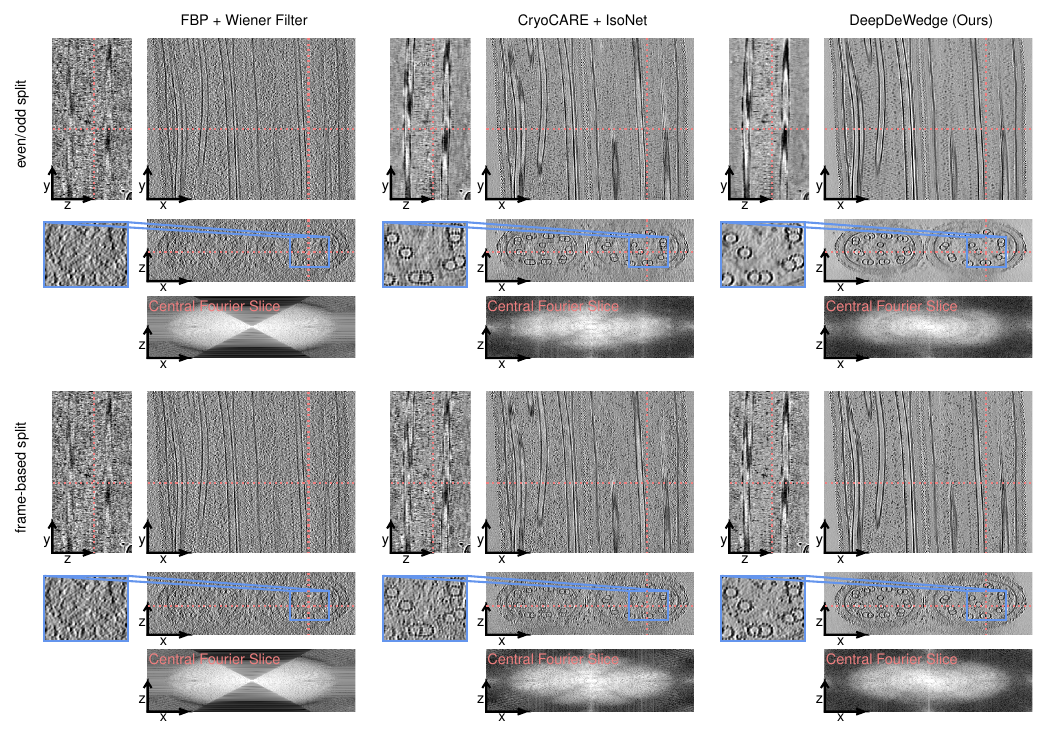}
    \caption{
    Slices through 3D reconstructions of the flagella of \textit{C. reinhardtii} when using different reconstruction methods. 
    The red lines in each slice indicate the positions of the remaining two slices.
    We also show the central x-z-slice through the logarithm of the magnitude of the Fourier transform of each reconstruction.
    }
    \label{fig:tomo110_vols}
\end{figure}

Next, we evaluate DeepDeWedge on another real-world dataset of tomograms of the flagella of \textit{C. reinhardtii}, which is the tutorial dataset for CryoCARE.
Since we observed above that IsoNet performs better when applying it to CryoCARE-denoised tomograms, we compare only to CryoCARE + IsoNet. In addition, we investigate the impact of splitting the tilt-series into even and odd projections versus using the frame-based split for DeepDeWedge and CryoCARE + IsoNet. More details on the data and the model fitting can be found in Supplementary Information \ref{appx:experiments:real:tomo110}.

\paragraph{Results.} The reconstructions obtained with all methods are shown in Figure \ref{fig:tomo110_vols}. We find that when using the even/odd-based split, CryoCARE + IsoNet produces a crisper reconstruction than DeepDeWedge (see zoomed-in region). This may be because the model inputs and targets of IsoNet stem from denoised FBP reconstruction of the full tilt series in which the information is more densely sampled than in the sub-tomograms of the even and odd FBP to the reconstructions used for fitting DeepDeWedge. 

\begin{figure}[!ht]
    \centering
    \includegraphics[width=\textwidth]{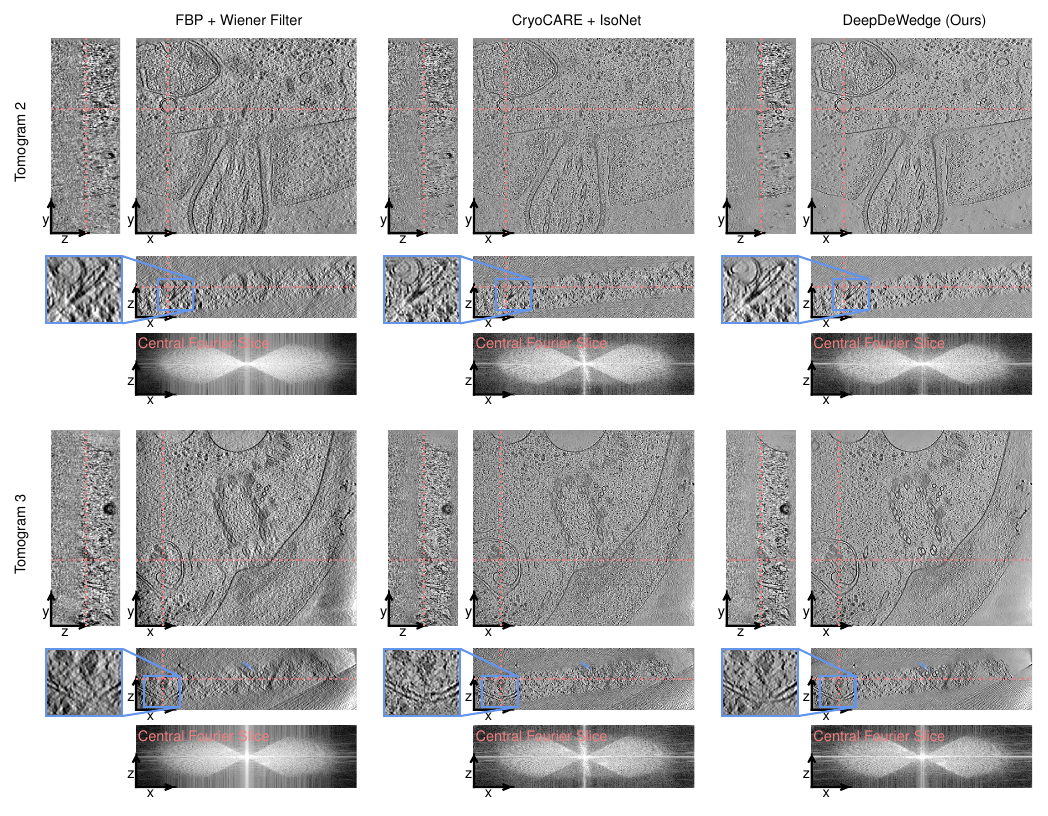}
    \caption{
    Slices through 3D reconstructions of tomograms showing the ciliary transit zone of \textit{C. reinhardtii}. 
    The red lines in each slice indicate the positions of the remaining two slices.
    We also show the central x-z-slice through the logarithm of the magnitude of the Fourier transform of each reconstruction.
    The row labels follow the naming convention of EMPIAR-11078. 
    }
    \label{fig:empiar11078_vols_frames}
\end{figure}

When using the frame-based splitting method, in which DeepDeWedge also operates on the more densely sampled FBP reconstruction, 
DeepDeWedge removes more noise than CryoCARE + IsoNet and produces higher contrast, which is most noticeable in the x-z-slice. Especially in background areas, the DeepDeWedge reconstruction has fewer high-frequency components and is smoother. Therefore, the CryoCARE + IsoNet reconstruction is slightly more faithful to the FBP reconstruction but is also noisier.

The central x-z-slices through the reconstructions' Fourier transforms' indicate that both CryoCARE + IsoNet and DeepDeWedge fill in most of the missing wedge.
Both methods fix the missing-wedge-caused distortions of the microtubules exhibited by the FBP reconstruction, as seen in the x-z slices. DeepDeWedge reconstructs more of the flagellas' outer parts.

\subsection{Experiments on the Ciliary Transit Zone of \textit{C. Reinhardtii}}
\label{sec:experiments:real:empiar11078}

Finally, we apply DeepDeWedge to an in situ dataset. We chose EMPIAR-11078 \cite{van2022situ}, which contains tilt series of the ciliary transit zone of \textit{C. reinhardtii}. The crowded cellular environment and low contrast and signal-to-noise ratio of the tilt series make EMPIAR-11078 significantly more challenging for denoising and missing wedge reconstruction than the two datasets from our previous experiments. More details on the data and the model fitting can be found in Supplementary Information \ref{appx:experiments:real:empiar11078}.

\paragraph{Results.} Slices through reconstructions of two tomograms obtained with FBP, CryoCARE + IsoNet and DeepDeWedge are shown in Figure \ref{fig:empiar11078_vols_frames}.
In the x-z- and z-y-planes, the DeepDeWedge reconstructions are more crisp and less noisy than those produced with CryoCARE + IsoNet. Especially in the x-z-plane, where the effects of the missing wedge are strongest, DeepDeWedge produces higher contrast than CryoCARE + IsoNet and removes more of the artifacts. 
Again, the CryoCARE + IsoNet reconstructions contain more high-frequency components and are closer to the FBP reconstruction, whereas the DeepDeWedge reconstructions are smoother and more denoised, especially in empty or background regions. 

Remarkably, as can be seen in the zoomed-in regions in the second row of Figure \ref{fig:empiar11078_vols_frames}, both CryoCARE + IsoNet and DeepDeWedge reconstruct parts of the sample that are barely present in the FBP reconstruction since they are perpendicular to the electron beam direction, which means that a large portion of their Fourier components are masked out by the missing wedge.

Note that both CryoCARE + IsoNet and DeepDeWedge appear less effective at reconstructing the missing wedge compared to the two experiments presented above. This is indicated by the central x-z-slices through the Fourier transforms of the reconstructions and is likely due to the challenging, crowded nature and low SNR of the data.

\subsection{Experiments on Synthetic Data}
\label{sec:experiments:synth}

Here, we compare DeepDeWedge to IsoNet, CryoCARE, and the CryoCARE + IsoNet combination on synthetic data to quantify our qualitative findings on real data. 

We used a dataset by Gubins et al.\ \cite{gubins2020shrec} containing 10 noiseless synthetic ground truth volumes with a spatial resolution of 10 \AA$^3$/voxel and a size of $179\times512\times512$ voxels. All volumes contain typical objects found in cryo-ET samples, such as proteins (up to 1,500 uniformly rotated samples from a set of 13 structures), membranes, and gold fiducial markers. 

For our comparison, we fitted models on the first three tomograms of the SHREC 2021 dataset. We used the Python library tomosipo \cite{hendriksen2021tomosipo} to compute clean projections of size $512 \times 512$ in the angular range $\pm 60^\circ$ with $2^\circ$ increment. From these clean tilt series, we generated datasets with different noise levels by adding pixel-wise independent Gaussian noise to the projections. We simulated three datasets with tilt series SNR $1/2$, $1/4$, and $1/6$, respectively. More details on the data and the model fitting can be found in Supplementary Information \ref{appx:experiments:synth}.

\begin{figure}[!t]
    \centering
    \includegraphics[width=\textwidth]{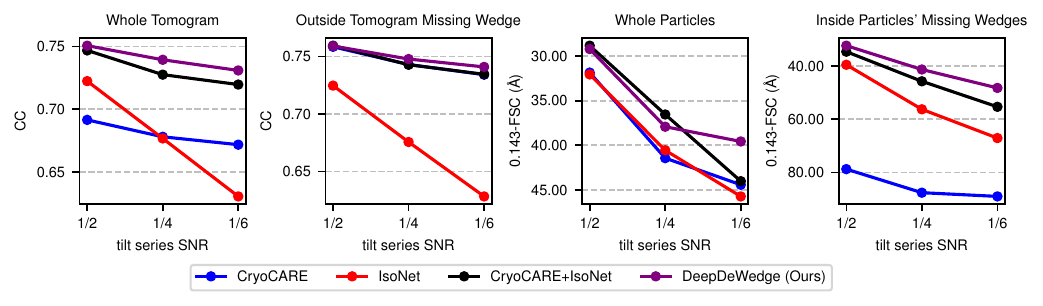}
    \caption{
    Comparison of our DeepDeWedge (with even/odd split) to the IsoNet, CryoCARE and the combination of IsoNet and CryoCARE for increasing noise levels on the tilt series images. We inverted the y-axes such that higher is better for all metrics.
    }
    \label{fig:synth_results:ours_vs_isonet}
\end{figure}

\begin{figure}[!t]
    \centering
    \includegraphics[width=\linewidth]{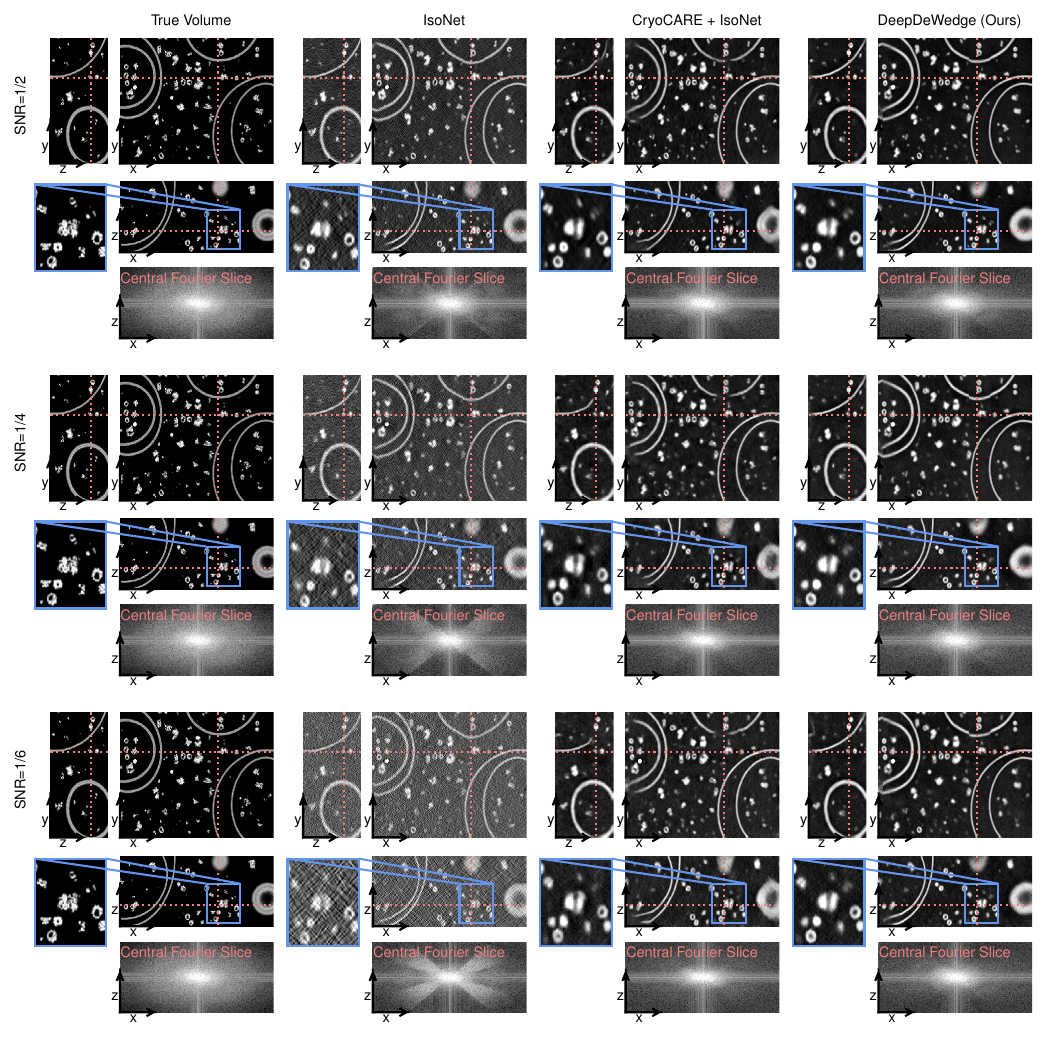}
    \caption{
    Slices through tomograms reconstructed with CryoCARE + IsoNet and DeepDeWedge (with even/odd split) on a tomogram from the synthetic SHREC 2021 dataset. 
    The red lines in each slice indicate the positions of the remaining two slices.
    We also show the central x-z-slice through the logarithm of the magnitude of the Fourier transform of each reconstruction.
    }
    \label{fig:synth_vols}
\end{figure}

\paragraph{Performance Metrics.} To measure the overall quality of a tomogram $\hat{\vol}$ obtained with any of the three methods, we calculated the
normalized correlation coefficient $\crosscorr(\recvol, \truevol)$ between the reconstruction $\hat{\vol}$ and the corresponding ground-truth $\truevol$, which is defined as 
    \begin{equation}
        \crosscorr(\recvol, \truevol) = \frac{\big\langle \recvol - \mean(\recvol), \truevol - \mean(\truevol) \big\rangle}{\norm{\recvol - \mean(\recvol)}_2 \norm{\recvol - \mean(\recvol)}_2}.
    \end{equation}
By definition, it holds that $0 \leq \crosscorr(\recvol, \truevol) \leq 1$, and the higher the correlation between reconstruction and ground truth, the better.
The correlation coefficient measures the reconstruction quality (both the denoising and the missing wedge reconstruction capabilities) of the methods. To isolate the denoising performance of a method from its ability to reconstruct the missing wedge, we also report the correlation coefficient between the refined reconstructions and the ground truth after applying a $60^\circ$ missing wedge filter to both of them. We refer to this metric as "CC outside the missing wedge" and it is used to compare to the denoising performance of CryoCARE, which does not perform missing wedge reconstruction. 

As a central application of cryo-ET is the analysis of biomolecules, we also report the resolution of all proteins in the refined tomograms. For this, we extracted all proteins from the ground truth and refined tomograms and calculated the average 0.143 Fourier shell correlation cutoff (0.143-FSC) between the refined proteins and the ground truth ones. 
The 0.143-FSC is commonly used in cryo-EM applications. Its unit is Angstroms, and it seeks to express up to which spatial frequency the information in the reconstruction is reliable. In contrast to the correlation coefficient, a lower 0.143-FSC value is better. 
To measure how well each method filled in the missing wedges of the structures, we also report the average 0.143-FSC calculated only on the true and predicted missing wedge data. We refer to this value as (average) "0.143-FSC inside the missing wedge".

\paragraph{Results.} Figure \ref{fig:synth_results:ours_vs_isonet} shows the metrics for decreasing SNR of the tilt series. All metrics suggest that DeepDeWedge yields higher-quality reconstructions than IsoNet, CryoCARE, and CryoCARE + IsoNet. 

CryoCARE achieves a lower correlation coefficient than IsoNet in the high-SNR regime, while the order is reversed for low SNR. A likely explanation is that for lower noise levels, the correlation coefficient is more sensitive to the missing wedge artifacts in the reconstructions. CryoCARE does not perform missing wedge reconstruction, so it has a lower correlation coefficient than IsoNet for higher SNR. For lower SNR, the correlation coefficient is dominated by the noise. Regarding denoising, we and others \cite{maldonado2023fourier} have observed that CryoCARE performs better than IsoNet, which is confirmed here by looking at the correlation coefficient outside the missing wedge. As expected, CryoCARE + IsoNet  combines the strengths of both methods. Compared to this combination, the overall quality of DeepDeWedge reconstructions is on par or better, depending on the noise level.

The FSC metrics in the second row of Figure \ref{fig:synth_results:ours_vs_isonet} indicate that the average resolution of the proteins in the refined tomograms is approximately the same for IsoNet and DeepDeWedge and that they perform similarly for missing wedge reconstruction.

\section{Discussion}
In this paper, we have proposed DeepDeWedge, which is a deep-learning-based method for denoising and missing wedge reconstruction in cryo-ET.

Compared to the state-of-the-art CryoCARE + IsoNet pipeline, DeepDeWedge removes more of the noise and, therefore, produces smoother reconstructions. While this occasionally results in the loss of some high-frequency details, DeepDeWedge produces overall cleaner reconstructions with higher contrast, especially of medium or large objects such as microtubules (Figure \ref{fig:tomo110_vols}, Figure \ref{fig:empiar11078_vols_frames}), cellular structures (Figure \ref{fig:empiar11078_vols_frames}, Figure \ref{fig:synth_vols}), and large proteins (Figure \ref{fig:empiar10045_vols}, Figure \ref{fig:empiar11078_vols_frames}). On the other hand, the CryoCARE + IsoNet results are more faithful to the FBP reconstructions, especially for high-frequency details, as can be seen in Figure \ref{fig:empiar11078_vols_frames}. Regarding missing wedge reconstruction, which can be mainly seen in the removal of missing wedge artifacts and the reduction of artificial elongations, we find that DeepDeWedge and CryoCARE + IsoNet have overall similar performance. Both methods perform best for small and medium-sized objects like proteins or microtubules, whereas large structures like vesicles (Figure \ref{fig:empiar11078_vols_frames}, Figure \ref{fig:synth_vols}) or the boundaries of the \emph{C. reinhardtii} flagella (Figure \ref{fig:tomo110_vols}) pose greater challenges. This could be a result of using sub-tomograms for model fitting, as sub-tomograms may not always contain all the information that is necessary for missing wedge reconstruction of objects larger than the sub-tomogram size. 

The data that DeepDeWedge fills in for the missing wedge is based on prior information about the data that the network learns during model fitting by predicting data that we artificially removed. 
However, it is important to keep in mind that the actual information contained in the missing wedge is irreversibly lost during tilt series acquisition. 

Therefore, we advise users of our method to be cautious, especially regarding reconstructing objects that are mostly perpendicular to the electron beam. Many of the Fourier components that correspond to such objects are contained in the missing wedge region.
Examples of objects that are almost completely masked by the missing wedge include thin membranes or elongated proteins perpendicular to the electron beam; we discuss a concrete example and how DeepDeWedge handles it in Supplementary Information \ref{appx:ortho_beam}.
It remains an open question to what extent missing-wedge-filled tomograms, whether based on deep learning or classical methods, can be trusted (see also Supplementary Information \ref{appx:hallucination}). 
To estimate the trustworthiness of a DeepDeWedge reconstruction, we recommend comparing it to a reconstruction obtained with a classical, prior-free method that is strongly data-consistent, such as FBP.

Based on these considerations, we see two main applications for DeepDeWedge: 
First, DeepDeWedge is an effective method to improve the direct interpretability of tomograms, which may enable discoveries that are prevented by high noise levels and strong missing wedge artifacts in raw FBP reconstructions.
Second, tomograms reconstructed with DeepDeWedge can be used as input for downstream tasks such as segmentation or particle picking and might improve the effectiveness of those downstream tasks. It has already been observed that denoised and/or missing-wedge-corrected tomograms can improve the performance of deep-learning-based particle pickers \cite{buchholz2019cryo, liu2022isotropic, hong2023cryo}. For subsequent steps, such as sub-tomogram averaging, one can then use the original raw tomograms that are not modified by the neural network.

\section*{Data Availability}

All tomograms shown in Figure \ref{fig:empiar10045_vols}, Figure \ref{fig:tomo110_vols}, Figure \ref{fig:empiar11078_vols_frames} and Figure \ref{fig:synth_vols} have been deposited to Figshare and are available at \url{https://figshare.com/articles/journal_contribution/Tomograms_shown_in_Figures_of_the_DeepDeWedge_Paper/26169538}.

The EMPIAR-10045 dataset used in this study is available in the EMPIAR database under accession code 10045 (\url{https://www.ebi.ac.uk/empiar/EMPIAR-10045/}).

The tilt series collected from the \emph{C. reinhardtii} flagella is available at \url{https://download.fht.org/jug/cryoCARE/Tomo110.zip}.

The EMPIAR-11078 dataset used in this study is available in the EMPIAR database under accession code 11078 (\url{https://www.ebi.ac.uk/empiar/EMPIAR-11078/}).

The SHREC 2021 dataset is available at \url{https://dataverse.nl/dataset.xhtml?persistentId=doi:10.34894/XRTJMA}.

\section*{Code Availability}
We provide an implementation and a tutorial of DeepDeWedge on GitHub: \url{https://github.com/MLI-lab/DeepDeWedge}.

\section*{Acknowledgements}

The authors would like to thank Tobit Klug, Kang Lin, Youssef Mansour, Ricardo Righetto, and Dave Van Veen for helpful discussions. The authors would like to thank Ricardo Righetto for help regarding tomogram reconstruction from the movie frames in EMPIAR-11078.

The authors are supported by the Institute of Advanced Studies at the Technical University of
Munich, the Deutsche Forschungsgemeinschaft (DFG, German Research Foundation) - 456465471,
464123524, the DAAD, the German Federal Ministry of Education and Research, and the Bavarian
State Ministry for Science and the Arts. The authors also acknowledge the financial support by the
Federal Ministry of Education and Research of Germany in the programme of ''Souveraen. Digital.
Vernetzt.''. Joint project 6G-life, project identification number: 16KISK002.

\newpage

\appendix

\renewcommand{\figurename}{Supplementary Figure}
\renewcommand{\thefigure}{\arabic{figure}}
\renewcommand{\thesection}{\arabic{section}}

\begin{center}
    \LARGE \textbf{Supplementary Information}
\end{center}
\vspace{0.5em} 

\section{Background on Self-Supervised Deep Learning Techniques}
\label{appx:background}
The DeepDeWedge loss is inspired by Noise2Noise \cite{lehtinen2018noise2noise} and Noisier2Noise \cite{moran2020noisier2noise} self-supervised learning. Here, we give a brief overview of the main ideas behind these two frameworks.


\subsection{Denoising with Noise2Noise}
\label{appx:background:Noise2Noise}

Noise2Noise is a framework for constructing a loss function that enables training a neural network for image denoising without ground-truth images. 
Neural networks for denoising are typically trained in a supervised fashion to map a noisy image to a clean one and thus require pairs of clean images and corresponding measurements. 

Noise2Noise-based methods also aim to train a neural network to map a noisy image to a clean one. Contrary to supervised learning, Noise2Noise assumes access to a dataset of pairs of noisy observations ${\color{red}\vec{y}^0_i} = \vec{x}^*_i + {\color{red} \vec{n}^0_i}$ and ${\color{blue} \vec{y}^1_i} = \vec{x}^*_i + {\color{blue} \vec{n}^1_i}$ for each ground-truth image $\vec{x}^*_i$. 
Moreover, it assumes that the noise terms ${\color{red} \vec{n}^0_i}$ and ${\color{blue} \vec{n}^1_i}$ are independent and that the noise ${\color{blue} \vec{n}^1_i}$ is zero-mean. 
Then, one can train a network $\network$ for denoising to map the noisy observation ${\color{red}\vec{y}^0}$ onto its counterpart ${\color{blue} \vec{y}^1}$. 
Formally, one can show that
    \begin{equation}
    \label{eq:noise2nosie}
        \expectation{{\color{red}\vec{y}^0}, {\color{blue} \vec{y}^1}} \Big[ \norm{\network({\color{red}\vec{y}^0}) - {\color{blue} \vec{y}^1}}_2^2 \Big] = \expectation{{\color{red}\vec{y}^0}, \vec{x}^*} \Big[ \norm{\network({\color{red}\vec{y}^0}) - \vec{x}^*}_2^2 \Big] + c,
    \end{equation}
where $c$ is a constant that is independent of the network weights $\nnweights$. Thus, a self-supervised Noise2Noise training objective based on samples approximates, up to an additive constant, the same underlying risk (left side of Equation \eqref{eq:noise2nosie}) as the supervised objective function. 
This approximation becomes better as the number of training examples increases and theoretically and empirically, networks trained with a Noise2Noise-like loss perform as well as if trained on sufficiently many examples~\cite{klug_ScalingLawsDeep_2023}. 

\subsection{Denoising and Recovering Missing Data with Noiser2Noise}
\label{appx:background:noisier2noise}
IsoNet's denoising approach is motivated by the Noisier2Noise framework \cite{moran2020noisier2noise}. While training a model for image denoising with a Noise2Noise objective requires paired noisy observations, a Noisier2Noise objective is based on a single noisy observation per image.  

\textbf{Noisier2Noise for Denoising:}
We assume that we have access to a dataset of noisy images. Contrary to Noise2Noise, we assume that we have only one single noisy observation $\vec{y} = \vec{x}^* + \vec{n}$ per ground-truth $\vec{x}^*$. The goal is to train a neural network for image denoising. For constructing a training objective with Noisier2Noise, we assume that the noise terms $\vec{n}$ that corrupt all images come from the same distribution and that we are able to sample from this distribution. We construct model inputs $\tilde{\vec{y}}$ by further corrupting the noisy images $\vec{y}$ with an additional noise term $\tilde{\vec{n}}$ sampled from the true noise distribution, i.e., $\tilde{\vec{y}} = \vec{y} + \tilde{\vec{n}}$. Moran et al.\ \cite{moran2020noisier2noise} showed that if we train a neural network $\network$ to map these noisier model inputs $\tilde{\vec{y}}$ onto their single noisy counterparts $\vec{y}$ using the squared L2-loss, one can use the final trained network to construct a denoiser for the double noisy images $\tilde{\vec{y}}$. 

\textbf{Noisier2Noise for Missing Data Recovery:}
We use ideas from Noisier2Noise for missing wedge reconstruction. To illustrate how Noisier2Noise can be used to predict missing data, consider the following setup: Assume we have a dataset of masked measurements $\vec{y} = \op{\mask} \vec{x}^*$. Here, $\op{\mask}$ denotes a random mask, i.e.,\ a diagonal matrix with entries in $\{0,1\}$, and $\vec{x}^*$ is an unknown ground-truth signal vector. Like above, the ground truths and masks are different for each measurement $\vec{y}$ and the dataset contains only the measurements and no ground truths.
In this setup, we want to train a neural network $\network$ to predict the data $\vec{x}^*$ from the measurement $\vec{y}$. To do this, we construct double-masked measurements $\tilde{\vec{y}} = \op{\noisierMask} \vec{y}$, where $\op{\noisierMask}$ is another random mask, and train the network to map the double-masked measurements $\tilde{\vec{y}}$ to their counterparts $\vec{y}$. 
It can be shown that, under certain assumptions on the random masks $\op{\mask}$ and $\op{\noisierMask}$, the final trained model can be used in an algorithm to recover the data missing in the singly masked measurements $\vec{y}$. As for denoising, the model estimates the missing data based on the doubly-masked measurement $\tilde{\vec{y}}$. For details, we refer readers to recent works by Millard and Chiew \cite{millard2022noisier2noise_mri, millard2022simultaneous}, who discussed Noisier2Noise for the recovery of missing data in accelerated magnetic resonance imaging.

\newpage
\section{Normalizing Tomograms and Sub-Tomograms}
\label{appx:normalization}

During model fitting, we apply Gaussian normalization to each model input sub-tomogram $\noisierVolZero{i,\anglesPhi_i}$ by subtracting a global mean $\mu$ and dividing by a global standard deviation $\sigma$, i.e., 
    \begin{equation}
    \label{eq:input_normalization}
        \noisierVolZero{i,\anglesPhi_i} \leftarrow \frac{\noisierVolZero{i,\anglesPhi_i} -\mu}{\sigma}.
    \end{equation}
We calculate the global mean $\mu$ and the global variance $\sigma^2$ from a set of $M$ randomly sampled model inputs as 
    \begin{equation}
    \label{eq:normalization_constants}
        \mu = \frac{1}{M} \sum_{i=1}^M \text{mean}\left(\noisierVolZero{i,\anglesPhi_i}\right), \qquad \sigma^2 = \frac{1}{M} \sum_{i=1}^M \text{var}\left(\noisierVolZero{i,\anglesPhi_i}\right).
    \end{equation}
    
In addition to this sub-tomogram level normalization, we found it beneficial for the performance and stability of DeepDeWedge to also normalize the full tomograms during the final refinement step. We now discuss our normalization approach, which is relevant mainly for the case when applying DeepDeWedge to multiple tilt series. 

Assume we want to obtain the DeepDeWedge reconstruction of a tilt series $\ts$. For this, we have to apply a fitted model to sub-tomograms of the FBP reconstructions $\FBP(\tsZero)$ and  $\FBP(\tsOne)$, where $\tsZero$ and $\tsOne$ are the result of splitting the tilt series $\ts$ using the even/odd or frame-based split. Before extracting sub-tomograms from the reconstructions $\FBP(\tsZero)$ and  $\FBP(\tsOne)$, we normalize them to have a mean $\mu_\ts$ and variance $\sigma^2_\ts$ via
    \begin{align}
        \overline{\FBP(\tsZero)} &= \frac{\FBP(\tsZero) - \text{mean}\big(\FBP(\tsZero)\big)}{\sqrt{\text{var}\big(\FBP(\tsZero)\big)}} \cdot \sigma_\ts + \mu_\ts,
        \\
        \overline{\FBP(\tsOne)} &= \frac{\FBP(\tsOne) - \text{mean}\big(\FBP(\tsOne)\big)}{\sqrt{\text{var}\big(\FBP(\tsOne)\big)}} \cdot \sigma_\ts + \mu_\ts.
    \end{align}
For refinement, we extract sub-tomograms from these normalized reconstructions and then normalize the sub-tomograms with the global mean $\mu$ and the global variance $\sigma^2$ as in Equation \eqref{eq:input_normalization}. We calculate the normalization constants $\mu_\ts$ and $\sigma_\ts$  as in Equation \eqref{eq:normalization_constants}, with the difference being that the sub-tomograms $\{\volZero{i}\}_{i=1}^M$ that underly the sub-tomograms $\{\noisierVolZero{i,\anglesPhi_i}\}_{i=1}^M$ are extracted exclusively from the FBP reconstruction $\FBP(\tsZero)$ of one half of the tilt series of interest $\ts$. This ensures that the voxels of the normalized reconstructions $\FBP(\tsZero)$ and $\FBP(\tsOne)$ lie in a range similar to that of the model inputs during model fitting, which are generated from the reconstruction $\FBP(\tsZero)$.

To achieve this, one could alternatively set $\mu_\ts = \mu$ and $\sigma^2_\ts = \sigma^2$, where the mean $\mu$ and the variance $\sigma^2$ are the global ones used during model fitting, which were calculated from a sample of all sub-tomograms used for model fitting. However, our approach described above gives better results if the mean and variance of the tilt series $\ts$ of interest are substantially different from the global mean $\mu$ and the global variance $\sigma^2$.

\newpage
\section{Proof of Proposition \ref{prop:selfsupervised_is_supervised}}
\label{appx:theory}

We divide the proof of Proposition \ref{prop:selfsupervised_is_supervised} into two steps. The first step, which we present as a lemma, is a result inspired by Noise2Noise and deals with the additive noise $\noiseOne$ on the model target. To facilitate notation, we assume from now on that all 3D objects are represented as column vectors in $\reals^{n}$ for $n=N^3$, and that the masks $\op{\mask}$ and $\op{\noisierMask}$ are random diagonal matrices with entries in $\{0,1\}$.

\begin{lem}
\label{lem:Noise2Noise}
    Assume that the noise term $\noiseOne$ is zero-mean and independent of the noise term $\noiseZero$ and the masks $\op{\mask}$ and $\op{\noisierMask}$. Then 
        \begin{equation*}
            \expectation{ \noiseOne} 
                \Bigg[
                \norm{ \big( \op{\noisierMask} \op{\mask} +2 \op{\noisierMask}^C \op{\mask}\big) \fourier \Big( \network \big( \noisierStrucZero{} \big) - \strucOne{}\Big)}_2^2
                \Bigg]
                = 
                \norm{ \big(\op{\noisierMask} \op{\mask} + 2 \op{\noisierMask}^C \op{\mask} \big) \fourier \Big( \network(\noisierStrucZero{}) - \truestruc{} \Big) }_2^2 
                +
                c_,
        \end{equation*}
    where $c > 0$ is a constant that does not depend on the weights $\nnweights$.
\end{lem}
\begin{proof}\renewcommand{\qedsymbol}{}[Proof of Lemma \ref{lem:Noise2Noise}]
We calculate
    \begin{align*}
         \norm{\op{\noisierMask} \op{\mask} \fourier \big( \network(\noisierStrucZero{}) - \strucOne{} \big) }_2^2 
         &=   \norm{\op{\noisierMask} \op{\mask} \fourier \big( \network(\noisierStrucZero{}) - \ifourier \op{\mask} \fourier \big( \truestruc{} + \noiseOne \big) \big) }_2^2  \\
         &=   \norm{\op{\noisierMask} \op{\mask} \fourier \big( \network(\noisierStrucZero{}) - \big( \truestruc{} + \noiseOne \big) \big) }_2^2 , \\
         &=   \norm{\op{\noisierMask} \op{\mask}  \fourier\big( \network(\noisierStrucZero{}) - \truestruc{} \big) }_2^2 \\
         &\qquad- 2 \left\langle  \op{\noisierMask}\op{\mask} \fourier \big( \network(\noisierStrucZero{}) - \truestruc{} \big), \op{\noisierMask}\op{\mask} \fourier \noiseOne \right\rangle \\
         &\qquad+ \norm{\op{\noisierMask} \op{\mask} \fourier \noiseOne}_2^2,
    \end{align*}
    where we used in the second step that $\op{\mask}^2 = \op{\mask}$, as $\op{\mask}$ is a diagonal matrix with entries in $\{0, 1\}$.
    As we assumed the noise $\noiseOne$ to be zero-mean and independent of the noise $\noiseZero$, and the masks $\op{\mask}$ and $\op{\noisierMask}$, it holds that
        \begin{equation*}
            \expectation{\noiseOne} \Bigg[  \left\langle \op{\noisierMask}\op{\mask} \fourier \big( \network(\noisierStrucZero{}) - \truestruc{} \big), \op{\noisierMask}\op{\mask} \fourier \noiseOne \right\rangle \Bigg] 
            = 0.
        \end{equation*}
    Therefore we get
        \begin{equation*}
            \expectation{\noiseOne} \Bigg[  \norm{\op{\noisierMask} \op{\mask} \fourier \big( \network(\noisierStrucZero{}) - \strucOne{} \big) }_2^2 \Bigg]
             = 
            \norm{\op{\noisierMask} \op{\mask} \fourier \big( \network(\noisierStrucZero{}) - \truestruc{} \big) }_2^2 
             + \expectation{\noiseOne} \Bigg[ \norm{\op{\noisierMask} \op{\mask} \fourier \noiseOne  }_2^2 \Bigg].
        \end{equation*}
    An analogous argument yields
        \begin{equation*}
            \expectation{\noiseOne} \Bigg[  \norm{\op{\noisierMask}^C \op{\mask} \fourier \big( \network(\noisierStrucZero{}) - \strucOne{} \big) }_2^2 \Bigg]
             = 
             \norm{\op{\noisierMask}^C \op{\mask} \fourier \big( \network(\noisierStrucZero{}) - \truestruc{} \big) }_2^2 
             + 
             \expectation{\noiseOne} \Bigg[ \norm{\op{\noisierMask}^C \op{\mask} \fourier \noiseOne  }_2^2 \Bigg].
        \end{equation*}
Using the fact that the joint masks $\op{\noisierMask}\op{\mask}$ and $\op{\noisierMask}^C \op{\mask}$ are orthogonal to each other and setting
        \begin{equation*}
            c = \expectation{\op{\mask}, \op{\noisierMask}, \noiseOne} \Bigg[ \norm{\big(\op{\noisierMask} \op{\mask} + 2 \op{\noisierMask}^C \op{\mask} \big) \fourier \noiseOne  }_2^2 \Bigg],
        \end{equation*}
    yields the desired result.
\end{proof}

\begin{proof}\renewcommand{\qedsymbol}{}[Proof of Proposition \ref{prop:selfsupervised_is_supervised}]
    For the sake of generality, we omit the assumption that the missing wedge masks $\op{\mask}$ and $\op{\noisierMask}$ are non-overlapping and show that
        \begin{equation}
        \label{eq:self_supervised_is_superfixed_fixed_angles}
            \expectation{\op{\mask}, \op{\noisierMask}, \noiseZero, \noiseOne} \Bigg[  \norm{ \big(\op{\noisierMask} \op{\mask} + 2 \op{\noisierMask}^C \op{\mask} \big) \fourier \Big( \network(\noisierStrucZero{}) - \strucOne{} \Big) }_2^2 \Bigg]
            =\expectation{\op{\mask}, \op{\noisierMask}, \noiseZero} \Bigg[  \norm{ \big(\id - \op{\noisierMask}^C \op{\mask}^C\big) \fourier \big( \network(\noisierStrucZero{}) - \truestruc{} \big) }_2^2 \Bigg] + c,
        \end{equation}
    where $c$ is the constant from Lemma \ref{lem:Noise2Noise}. 
    Here, the mask $\id - \op{\noisierMask}^C \op{\mask}^C$ zeros out all Fourier components that are contained in both missing wedges. 
    If the wedges are non-overlapping, this joint mask reduces to the identity $\id$, and we obtain the result stated in Proposition \ref{prop:selfsupervised_is_supervised}.
    
    By Lemma \ref{lem:Noise2Noise}, in order to show Equation \eqref{eq:self_supervised_is_superfixed_fixed_angles}, it suffices to show that
        \begin{align}
        \label{eq:toshowmmn}
            \expectation{\op{\mask}, \op{\noisierMask}, \noiseZero} \Bigg[  \norm{ \big(\op{\noisierMask} \op{\mask} + 2 \op{\noisierMask}^C \op{\mask} \big) \fourier \Big( \network(\noisierStrucZero{}) - \truestruc{} \Big) }_2^2 \Bigg]
            = \expectation{\op{\mask}, \op{\noisierMask}, \noiseZero} \Bigg[  \norm{ \big(\id - \op{\noisierMask}^C \op{\mask}^C\big) \fourier \big( \network(\noisierStrucZero{}) - \truestruc{} \big) }_2^2 \Bigg].
        \end{align}
    
    Now we analyze the left-hand side of this equation. We fix the noise $\noiseZero$ and calculate
        \begin{align*}
            &
            \expectation{\op{\mask}, \op{\noisierMask}} \left[
            \norm{2\op{\noisierMask}^C \op{\mask} \fourier \big( \network(\noisierStrucZero{}) - \truestruc{} \big) }_2^2 
            \right] 
            \\
            &= \expectation{\op{\mask}, \op{\noisierMask}} \left[
            \norm{\op{\noisierMask}^C \op{\mask} \fourier \big( \network(\noisierStrucZero{}) - \truestruc{} \big) }_2^2 
            \right]
            +
            \expectation{\op{\mask}, \op{\noisierMask}} \left[
            \norm{\op{\noisierMask}^C \op{\mask} \fourier \big( \network(\noisierStrucZero{}) - \truestruc{} \big) }_2^2 
            \right] \\
            &= \expectation{\op{\mask}, \op{\noisierMask}} \left[
            \norm{\op{\noisierMask}^C \op{\mask} \fourier \big( \network(\noisierStrucZero{}) - \truestruc{} \big) }_2^2 
            \right]
            +
            \expectation{\op{\mask}, \op{\noisierMask}} \left[
            \norm{\op{\mask}^C \op{\noisierMask} \fourier \big( \network(\noisierStrucZero{}) - \truestruc{} \big) }_2^2 
            \right] \\
            &=
            \expectation{\op{\mask}, \op{\noisierMask}} \left[
            \norm{ (\op{\noisierMask}^C \op{\mask} +  \op{\mask}^C \op{\noisierMask}) \fourier \big( \network(\noisierStrucZero{}) - \truestruc{} \big) }_2^2 
            \right].
        \end{align*}
    For the second equation, we used that $\prob(\op{\mask},\op{\noisierMask}) =  \prob(\op{\noisierMask},\op{\mask})$ and that the model input $\noisierStrucZero{}$ depends on the masks $\op{\mask}$ and $\op{\noisierMask}$ only through their product $\op{\mask} \op{\noisierMask}$, for which the order of the masks does not play a role, thus the role of the two masks is exchangeable.
    For the last equation, we used that the masks  $\op{\noisierMask}^C\op{\mask}$ and $\op{\mask}^C \op{\noisierMask}$ are orthogonal.

    Using this equation and the fact that $\op{\noisierMask} \op{\mask}$, $\op{\noisierMask}^C \op{\mask}$ and $\op{\mask}^C \op{\noisierMask}$ are orthogonal to each other, the left-hand side of Equation~\ref{eq:toshowmmn} becomes 
        \begin{align*}
     \expectation{\op{\mask}, \op{\noisierMask}, \noiseZero} \Bigg[  \norm{ \big(\op{\noisierMask} \op{\mask} + \op{\noisierMask}^C \op{\mask} + \op{\mask}^C \op{\noisierMask} \big) \fourier \big( \network(\noisierStrucZero{}) - \truestruc{} \big) }_2^2  \Bigg]
     \stackrel{(*)}{=}
     \expectation{\op{\mask}, \op{\noisierMask}, \noiseZero} \Bigg[  \norm{  \big(\id - \op{\noisierMask}^C\op{\mask}^C\big) \fourier \big( \network(\noisierStrucZero{}) - \truestruc{} \big) }_2^2  \Bigg],
        \end{align*}
   where the equality $(*)$ holds because
        \begin{align*}
            \op{\noisierMask} \op{\mask} + \op{\noisierMask}^C \op{\mask} + \op{\mask}^C \op{\noisierMask} 
            &= \op{\noisierMask}\op{\mask} + (\id - \op{\noisierMask})\op{\mask} + (\id - \op{\mask}) \op{\noisierMask} \\
            &= \op{\noisierMask}\op{\mask} + \op{\mask}-\op{\noisierMask}\op{\mask} + \op{\noisierMask} -\op{\mask}\op{\noisierMask} \\
            &=  \op{\mask}- \op{\noisierMask}\op{\mask} + \op{\noisierMask}  \\
            &= \id - \op{\noisierMask}^C\op{\mask}^C.
        \end{align*}
    This concludes the proof of Equation~\ref{eq:toshowmmn} and Proposition \ref{prop:selfsupervised_is_supervised}.

\end{proof}

\newpage
\section{Details on Main Experiments}
\label{appx:experiments}
Here, we describe the technical details of all the experiments presented in the main paper. 

\subsection{Experiments on Purified \textit{S. Cerevisiae} 80S Ribosomes}
\label{appx:experiments:real:empiar10045}
The EMPIAR-10045 dataset
contains 7 tilt series collected from samples of purified \textit{S. Cerevisiae} 80S Ribosomes. All tilt series are aligned and consist of 41 projections collected at tilt angles from  $-60^\circ$ to $+60^\circ$ with a $3^\circ$ increment. 

\begin{itemize}
    \item \textbf{FBP and sub-tomograms:} We performed all FBP reconstructions with a ramp filter in Python using the library tomosipo \cite{hendriksen2021tomosipo}. After reconstruction, we downsampled all tomograms by a factor of 6 using average pooling, which resulted in a final voxel size of 13.02 \AA. To these tomograms, we applied IsoNet's CTF deconvolution routine with parameters as described by Liu et al.\ \cite{liu2022isotropic} in their paper. For model fitting, we extracted sub-tomograms of shape $80 \times 80 \times 80$. As wide regions of the tomograms contain only ice and no ribosomes, we used IsoNet's mask generation tool with the default parameters to produce a mask of the non-empty regions. After extracting the sub-tomograms, we selected only those that contain at least 40\% sample according to the mask. For the refinement with the fitted model, we also used sub-tomograms of shape  $80 \times 80 \times 80$ voxels but extracted them without masking and with an overlap of 40 voxels.

    \item \textbf{Details on IsoNet:} We fitted IsoNet for 45 iterations (10 epochs per iteration) on 423 sub-tomograms. 
    We used the default noise schedule for fitting, which starts adding noise with level 0.05 in iteration 11 and then increases the noise level by 0.05 every 5 iterations.  We set the \texttt{noise\_mode} parameter, which determines the distribution of the additional noise used for Noisier2Noise-like denoising, to \texttt{noFilter}, which corresponds to Gaussian noise. We also tried setting the noise mode to \texttt{ramp}, as the tomograms were reconstructed with FBP with a ramp filter, but this gave worse results.
    In their paper, Liu et al.\ \cite{liu2022isotropic} fitted IsoNet for 30 iterations using similarly many sub-tomograms and the same number of epochs. However, we found that 45 iterations gave visually more appealing results compared to the images shown in the IsoNet paper. 

    \item \textbf{Details on CryoCARE + IsoNet:} We first fitted CryoCARE 
    on 1291 sub-tomograms 
    for 200 epochs, which was approximately when the validation loss started to increase. Next, we fitted IsoNet for 50 iterations (10 epochs per iteration) on 418 sub-tomograms, and monitored the validation loss. Iteration 32 achieved the lowest validation loss, so we chose this model to produce the final reconstructions.

    \item \textbf{Details on DeepDeWedge:} We fitted the model on 418 sub-tomograms until the loss on a hold-out validation set did not improve anymore, which took about 1500 epochs. However, we also obtained reconstructions that were visually similar to the one shown in Figure \ref{fig:empiar10045_vols} around the 500 epoch mark. 
\end{itemize}

\subsection{Experiments on Flagella of \textit{C. Reinhardtii}}
\label{appx:experiments:real:tomo110}
The dataset is a single tilt series collected from the flagella of \textit{C. reinhardtii}. The projections were collected at angles from $-65^\circ$ to $+65^\circ$ with $2^\circ$ increment. Each projection was acquired using dose fractionation with $10$ frames per tilt angle. Each frame has pixel size $2.36$~\AA. 

\begin{itemize}
    \item \textbf{FBP and sub-tomograms:}
    We followed the steps described in the CryoCARE GitHub repository (\url{https://github.com/juglab/cryoCARE_T2T/tree/master/example}) to obtain a CTF-corrected tomogram from the tilt series with IMOD \cite{kremer1996computer}. We downsampled the tomogram by a factor of 6 using average pooling, which resulted in a voxel size of $14.13$ \AA. As IsoNet is optimized for a missing wedge of $60^\circ$, we artificially widened the missing wedge from the original $50^\circ$ to $60^\circ$ by multiplying the reconstruction with a missing wedge filter in the Fourier domain for a fair comparison. 
    
    We again extracted sub-tomograms with shape $96 \times 96 \times 96$ for model fitting. For the final refinement, we used the same sub-tomogram size extracted with an overlap of 32 voxels.
    
    \item \textbf{Details on CryoCARE + IsoNet:}
    We first fitted CryoCARE for 2000 epochs on 252 sub-tomograms using early stopping on the self-supervised denoising loss on a hold-out validation set for model selection. 
    Next, we fitted IsoNet for 200 iterations (10 epochs per iteration) on 156 denoised sub-tomograms
    Interestingly, for the even/odd tilt series split, we found that the CryoCARE + IsoNet reconstruction looked better when using a CryoCARE model after 200 epochs rather than the one found with early stopping. In the main paper, we show these visually more appealing results.
    
    \item \textbf{Details on DeepDeWedge:}
    We fitted the model for 2000 epochs on 144 sub-tomograms.
\end{itemize}

\subsection{Experiments on the Ciliary Transit Zone of \textit{C. Reinhardtii}}
\label{appx:experiments:real:empiar11078}

For our comparison, we chose Tomograms 2, 3, and 8, since their tilt series were approximately collected at the angular range from $-60^\circ$ to $+60^\circ$  (increment: $2^\circ$), for which IsoNet was optimized. We fitted DeepDeWedge and the CryoCARE denoiser for CryoCARE + IsoNet using the frame-based split of the tilt-series data. 


\begin{itemize}
    \item \textbf{FBP and sub-tomograms:}
    We used IMOD to reconstruct the tomograms from the raw movie frames. First, we binned the pre-aligned tilt series 6 times, which resulted in a pixel size of 20.52 \AA. Next, we performed phase flipping based on the provided defocus values for CTF correction and reconstructed the tomograms via FBP.
    
    We extracted sub-tomograms with shape $96 \times 96 \times 96$ for model fitting. Again, we used masks generated with IsoNet to exclude empty areas from sub-tomogram extraction; each sub-tomogram had to contain at least 30\% sample according to the mask. 
    
    For the final refinement, we again used the same sub-tomograms of size $96 \times 96 \times 96$ extracted with an overlap of 32 voxels.
    
    \item \textbf{Details on CryoCARE + IsoNet:} We first fitted CryoCARE on 210 sub-tomograms for 2000 epochs, again using early stopping on the validation loss for model selection. Next, we fitted IsoNet for 300 iterations (10 epochs per iteration) on 264 denoised sub-tomograms. 

    \item \textbf{Details on DeepDeWedge:}  We fitted the model for 3000 epochs on 194 sub-tomograms.
\end{itemize}

\subsection{Experiments on Synthetic Data}
\label{appx:experiments:synth}

In an effort to make the comparison between the methods as fair as possible, we tried to eliminate implementation details wherever possible. For example, we used our own custom implementation for sub-routines like sub-tomogram extraction and reassembling sub-tomograms into full-sized volumes for all methods. Moreover, to be consistent with IsoNet's default parameters, we used dropout during model fitting in DeepDeWedge, although we found that it slightly deteriorated its performance. 

\begin{itemize}
    \item \textbf{FBP and sub-tomograms:}
    We performed FBP using tomosipo \cite{hendriksen2021tomosipo}, this time with a Hamming-like filter (see below).
    We fitted all models for all methods using 150 sub-tomograms of size $96\times 96 \times 96$. To refine the FBP reconstructions after model fitting, we applied the final fitted networks to sub-tomograms of the same size, using an overlap of 32 voxels. 
    
    \item \textbf{Details on CryoCARE: }
    Despite the use of dropout, we observed overfitting, so we applied early stopping on a subset of the ground truth data to get an impression of the best-case performance. Depending on the noise level, the optimal early stopping lay between epochs 300 and 900. 
    
    \item \textbf{Details on IsoNet:} 
    As our simulated dataset does not contain CTF effects, we did not perform the CTF deconvolution preprocessing step, which is usually performed to increase the contrast of the FBP reconstructions and improve IsoNet's performance. We compensated for this by using the Hamming-like filter for FBP, which improves the contrast compared to the standard ramp filter. We fitted IsoNet for 50 iterations with 20 epochs per iteration. We again followed the default noise schedule described above and set the noise mode to \texttt{hamming}.
    
    As there is no principled way to determine when to stop IsoNet fitting or how much noise to add, we did the following to obtain a very strong baseline: We evaluated the IsoNet reconstructions after every iteration using all comparison metrics and chose the best result for each metric, which yields best-case performance. We emphasize that this approach is not possible in practice where ground truth is not available. The optimal performance of each model occurred for every SNR between iterations 30 and 40, i.e., between 600 and 800 epochs. 

    \item \textbf{Details on CryoCARE + IsoNet:} 
    For each noise level, we denoised all volumes used for model fitting with the early-stopped CryoCARE models we fitted before. Next, we fitted IsoNet for 50 iterations with 10 epochs per iteration and without the builtin Noisier2Noise-like denoiser. Finally, we again calculated all metrics for each iteration and chose the best-performing ones.
 
    \item \textbf{Details on DeepDeWedge:} We fitted the network for 1000 epochs. At this point, the validation metrics were still improving, but only slightly.
    
\end{itemize}

\newpage
\section{IsoNet Reconstructions with Built-In Denoiser}
\begin{figure}[H]
    \centering
    \includegraphics{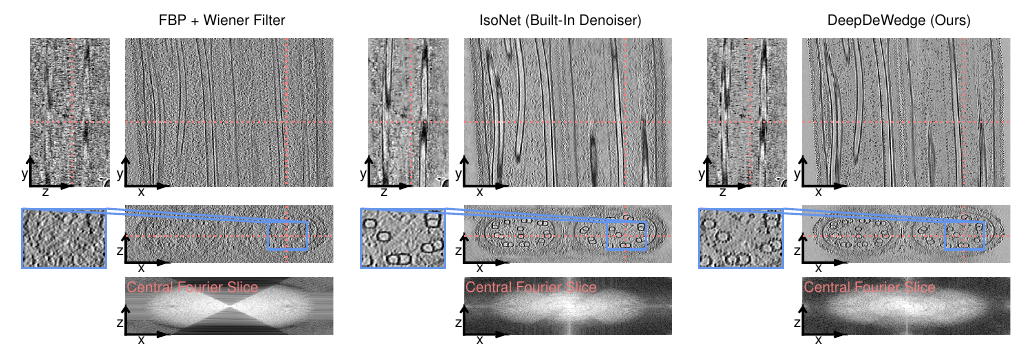}
    \caption{Reconstructions of the \textit{C. reinhardtii} flagella. IsoNet was fitted with its built-in Noisier2Noise-like denoisier rather than pre-denoising the FBP reconstruction with CryoCARE, resulting in a more noisy reconstruction.}
    \label{fig:vols_tomo110_isonet_nocare}
\end{figure}

\newpage
\section{Further Experiments}
\label{appx:ablations}

If not explicitly stated otherwise, we fitted all models on simulated tilt series from the first 3 tomograms of the SHREC 2021 Dataset 
with SNR $1/4$. We performed only one run for each experiment. All remaining details regarding the model and optimizer are as in the Methods Section of the main document.

\subsection{Applying DeepDeWedge to Multiple Tilt Series Simultaneously}
\label{appx:ablations:1_vs_3}

As stated in the main body, and as done in some of our experiments, DeepDeWedge can be applied to multiple tilt series from similar samples simultaneously. 
Here, we discuss how this impacts performance compared to applying DeepDeWedge to each tilt series separately. We compare the performance of DeepDeWedge on the first 3 tomograms of the SHREC 2021 dataset for two scenarios:
    \begin{itemize}
        \item Collective fitting: We fitted DeepDeWedge for 1500 epochs on 150 sub-tomograms extracted from the FBP reconstructions of the three tilt series. 
        
        \item Individual fitting: For each tilt series, we fitted one model for 3000 epochs on 50 sub-tomograms from the FBP reconstruction. We applied each model to the one tilt series used for its fitting.
    \end{itemize}

We compared the best-case performances of collective fitting to individual fitting. Considering the average correlation coefficients with respect to the three ground truth tomograms, both approaches perform the same.
However, all synthetic ground truth tomograms are very similar and densely packed with proteins. We expect collective fitting to be beneficial if the individual samples are  similar but sparse, i.e. contain little signal of interest. For example, we found that collective fitting on all seven EMPIAR-10045 tilt series gave visually more appealing reconstructions than individual fitting.
In case the samples are of different types, e.g. a combination of purified proteins or viruses and in-situ data, we recommend individual fitting. 

\subsection{Influence of Sub-Tomogram Size on Performance}
\label{appx:ablations:subtomogram_size}

\begin{figure}[H]
    \centering
    \includegraphics[width=\textwidth]{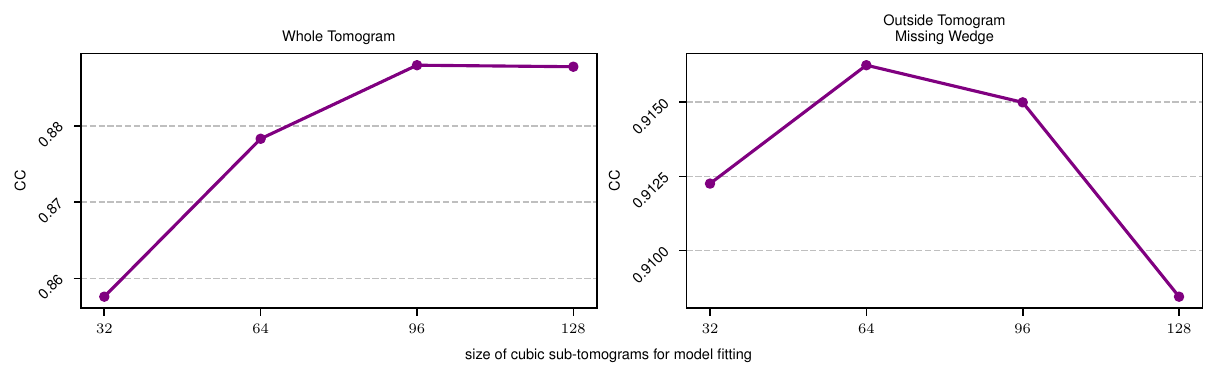}
    \caption{Performance metrics for different sub-tomogram sizes used for model fitting.}
    \label{fig:ablation_chunk_size}
\end{figure}

The size of the sub-tomograms used for model fitting is a hyperparameter of DeepDeWedge, and here we investigate how it affects the performance. To this end, we repeated the experiment on synthetic data described for increasing the size of sub-tomograms. We scaled the total number of sub-tomograms for each run such that the number of voxels of the sub-tomogram fitting dataset is approximately constant. As for model fitting, we have to rotate the sub-tomograms, the sub-tomograms we extract have to be larger than the ones we actually use for fitting in order to avoid having to use padding. As a result, the maximum sub-tomogram size we could use for model fitting is $128\times 128 \times 128$. 

We fitted all models for 1000 epochs. During fitting, starting from epoch 500, we evaluated the models on all three tilt series every 100 epochs. We report the metrics for the epoch in which the highest correlation coefficient between reconstruction and ground-truth was achieved. 

Figure \ref{fig:ablation_chunk_size} suggests that larger sub-tomograms for model fitting yield overall better performance than smaller ones. Note that the y-axis scales of the two plots are different and that the differences in the correlation coefficient outside the missing wedge (right plot) are very minor compared to the overall correlation coefficient (left plot).

\newpage
\section{Reconstructing Objects Perpendicular to the Electron Beam} 
\label{appx:ortho_beam}
\begin{figure}[H]
    \centering
    \includegraphics[scale=0.97]{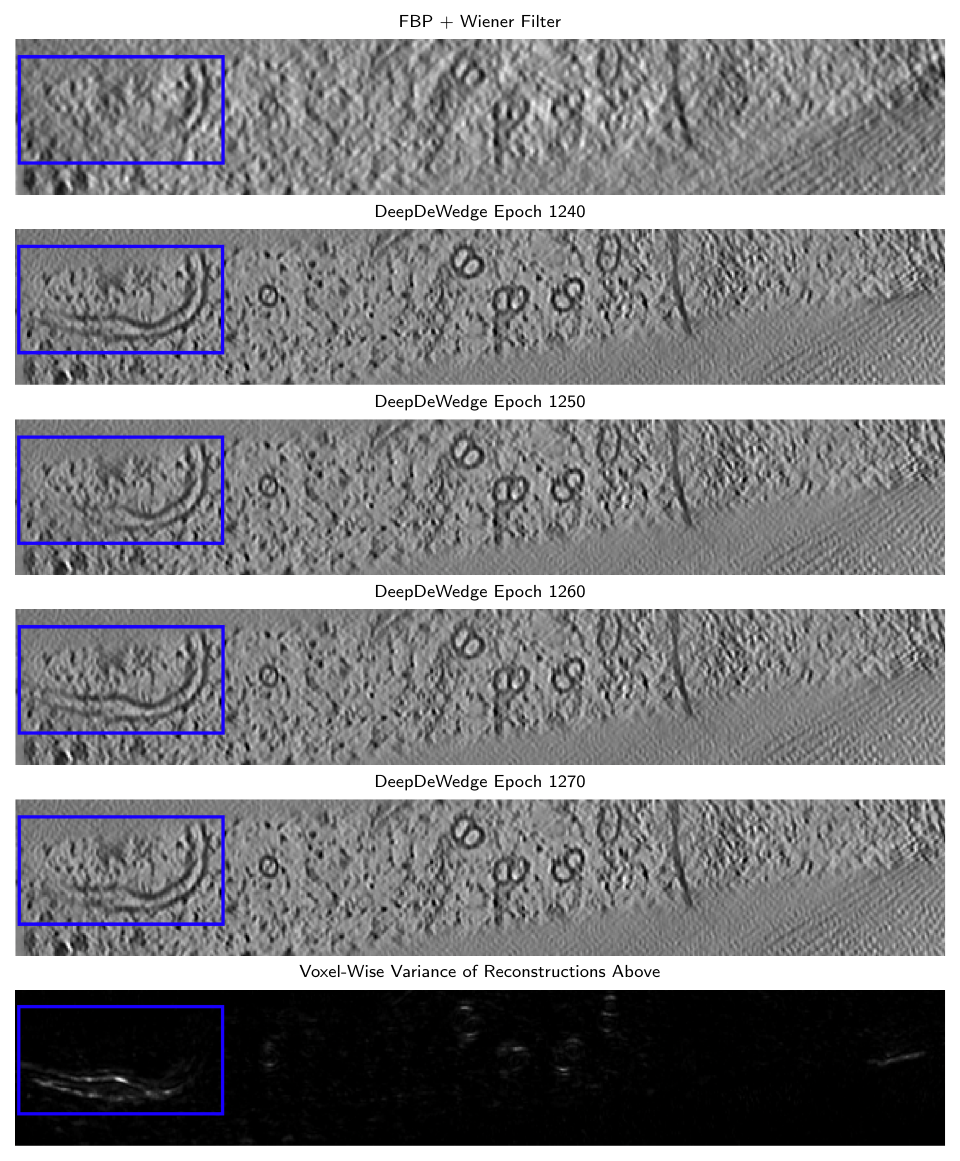}
    \caption{Slices through reconstructions of Tomogram 3 from EMPIAR-11078. 
    }
    \label{fig:appx:ortho_beam}
\end{figure}
Reconstructing objects perpendicular to the electron beam is difficult as many of their Fourier components lie inside the missing wedge region and are, therefore, not measured during tilt series acquisition. Nevertheless, we found that both DeepDeWedge and IsoNet can sometimes produce meaningful reconstructions of parts that are perpendicular to the electron beam, as can be seen, for example, in the x-z slice through the reconstructions of Tomogram 3 of EMPIAR-11078 shown in Figure \ref{fig:empiar11078_vols_frames}. 

Figure \ref{fig:appx:ortho_beam} displays an extreme case of a structure that is almost perfectly perpendicular to the electron beam. 
As the prediction in the blue box is based on very little data, and as there is no ground truth, it is unclear whether the predicted structure is correct. This is also indicated by the high voxel-wise variance over the course of different epochs in that area.

\newpage
\section{A Note on Hallucinations}
\label{appx:hallucination}

\begin{figure}[h!]
    \centering
    \includegraphics[width=\textwidth]{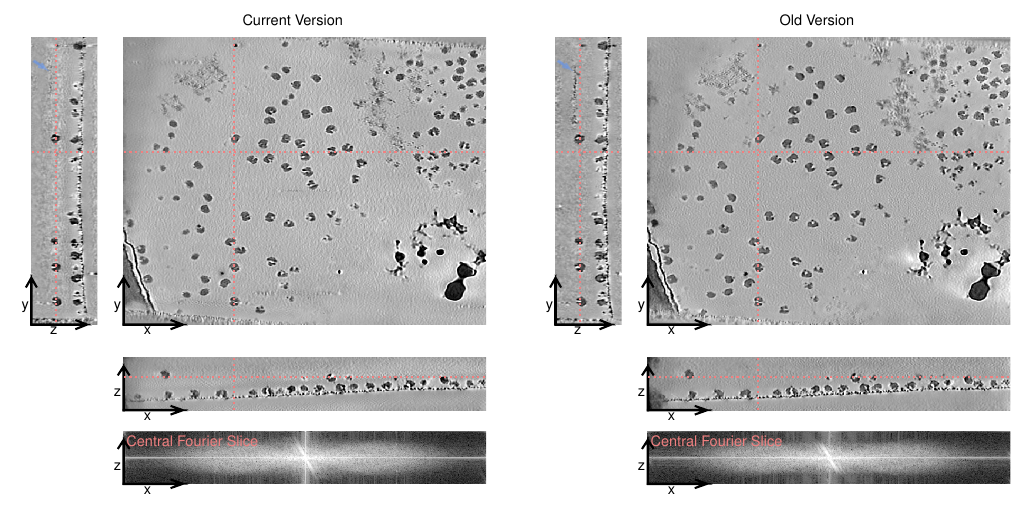}
    \caption{Reconstructions of Tomogram 5 of EMPIAR-10045 obtained with the current version and an old version of DeepDeWedge.}
    \label{fig:appx:hallucination_empiar10045}
\end{figure}

We found that an earlier version of DeepDeWedge produces hallucinations or overpronounced details in the reconstructed tomograms. In Figure \ref{fig:appx:hallucination_empiar10045}, we show the reconstruction of Tomogram 5 of EMPIAR 10045 obtained with this older version of DeepDeWedge, and the current version (left panel). We observe that with the old version, the reconstruction shows some unexpectedly strong densities at the air-water interface, which we marked with blue arrows. We do not observe such hallucinations in the volume reconstructed with the current version of DeepDeWedge.

We believe that the hallucinations shown in Figure \ref{fig:appx:hallucination_empiar10045} result from a mismatch between the model inputs used during model fitting and the inputs used for final reconstruction in the older version of the algorithm. 
In the older version of DeepDeWedge, the final refinement step, i.e. Step 3, was slightly different than the one presented in the Results section of the main paper: We used the FBP reconstruction of the full, non-splitted tilt series as input to the fitted model. Moreover, we did not progressively fill in the missing wedges of the model inputs during fitting in the previous version. This resulted in said mismatch: During fitting, the model received sub-tomograms from the FBP reconstructions of the splitted tilt series with an additional missing wedge as input. These are noisier and have stronger artifacts than the sub-tomograms of the full FBP reconstruction with only the original missing wedge. In the current version of DeepDeWedge, we try to minimize this mismatch by progressively filling the missing wedge of the model inputs during fitting and refining the FBP reconstructions of the splitted tilt series. 

Finally, we note that we observed similar hallucinations as the ones shown in Figure \ref{fig:appx:hallucination_empiar10045} with our re-implementation of Cryo-CARE when applying the fitted model to the FBP reconstruction of the full tilt series.

\end{document}